\documentclass{article}

\usepackage{microtype}
\usepackage{graphicx}
\usepackage{subcaption}
\usepackage{booktabs}

\usepackage{amsthm, amssymb}
\usepackage{mathtools}
\usepackage[algo2e, vlined, ruled, linesnumbered]{algorithm2e}
\usepackage{tabularx}
\usepackage{multirow}

\usepackage{hyperref}
\usepackage{cleveref}
 \crefname{section}{Section}{Sections}
 \crefname{theorem}{Theorem}{Theorems}
 \crefname{lemma}{Lemma}{Lemmas}
 \crefname{equation}{Equation}{Equations}
 \crefname{proposition}{Proposition}{Propositions}
 \crefname{claim}{Claim}{Claims}
\crefname{appendix}{Appendix}{Appendices}
 \crefname{algorithm}{Algorithm}{Algorithms}
 \crefname{figure}{Figure}{Figures}
 \crefname{table}{Table}{Tables}
 \crefname{remark}{Remark}{Remarks}
 \crefname{definition}{Definition}{Definitions}
 \crefname{equation}{Equation}{Equations}
 \crefname{corollary}{Corollary}{Corollaries}
 \crefname{observation}{Observation}{Observations}
\usepackage{bm}
\newcommand{\R}{\mathbb{R}}
\newcommand{\N}{\mathbb{N}}
\newcommand{\Z}{\mathbb{Z}}
\newcommand{\V}{\mathcal{V}}
\newcommand{\X}{\mathcal{X}}
\newcommand{\x}{\mathbf{x}}
\newcommand{\y}{\mathbf{y}}

\newcommand{\ba}{\mathbf{a}} 
\newcommand{\bb}{\mathbf{b}}  

\newcommand{\eb}{\mathbf{e}}

\newcommand{\rhob}{\bm{\rho}}
\newcommand{\xib}{\bm{\xi}}
\newcommand{\thetab}{\bm{\theta}}
\newcommand{\MV}{\mathcal{M}_{\mathcal{V}}}
\newcommand{\PMV}{\mathcal{P}(\mathcal{M}_{\mathcal{V}})}

\newcommand{\EV}{\mathbb{E}}

\newcommand{\q}{q(\x,\rhob)}
\newcommand{\Rb}{\mathcal{R}_{\mathcal{V}}}
\newcommand{\MA}{\mathcal{A}}
\newcommand{\MB}{\mathcal{B}}
\DeclarePairedDelimiter\ceil{\lceil}{\rceil}
\DeclarePairedDelimiter\floor{\lfloor}{\rfloor}
\DeclarePairedDelimiter\innerprod{\langle}{\rangle}

\DeclarePairedDelimiterX{\infdivx}[2]{(}{)}{%
  #1\;\delimsize\|\;#2%
}
\newcommand{\kl}{KL\infdivx}

\newcommand{\Int}{\int\limits}

\DeclareMathOperator*{\argmax}{arg\,max}

\newcommand{\groundset}{\ensuremath{\mathcal{V}}}

\newcommand{\sete}[3]{\mathbf #1|_{#2} (#3)}
\newcommand{\epe}[2][]{\underset{#1}{\mathbb E}\left[#2\right]}  

\SetCommentSty{mycommfont}
\SetKwInOut{Input}{input}
\SetKwInOut{Output}{output}

\usepackage{xspace}
\newcommand{\algname}[1]{{\texttt{#1}}}

\newcommand{\shrunkenfw}{\algname{Shrunken FW}\xspace}
\newcommand{\twophasefw}{\algname{Two-Phase FW}\xspace}

\newcommand{\blockca}{\algname{Block CA}\xspace}

\newtheorem{proposition}{Proposition}
\newtheorem*{proposition*}{Proposition}
\graphicspath{{./figures/}}

\usepackage[accepted]{icml2020}

\icmltitlerunning{Provable Variational Inference for Probabilistic Integer Submodular Models}

\begin{document}

\twocolumn[
\icmltitle{From Sets to Multisets: Provable Variational  Inference for\\ Probabilistic Integer Submodular Models}



\icmlsetsymbol{equal}{*}

\begin{icmlauthorlist}
\icmlauthor{Aytunc Sahin}{ethz}
\icmlauthor{Yatao Bian}{tencent}
\icmlauthor{Joachim M. Buhmann}{ethz}
\icmlauthor{Andreas Krause}{ethz}
\end{icmlauthorlist}

\icmlaffiliation{ethz}{Department of Computer Science, ETH Zurich, Zurich, Switzerland}
\icmlaffiliation{tencent}{Tencent AI Lab, Shenzhen, China 518057}

\icmlcorrespondingauthor{Aytunc Sahin}{aytunc.sahin@inf.ethz.ch}

\icmlkeywords{Machine Learning, ICML, Submodularity, Probabilistic Submodular Models, Variational Inference}

\vskip 0.3in
]



\printAffiliationsAndNotice{}  

\begin{abstract}
\looseness -1 Submodular functions have been studied extensively in machine learning and data mining. In particular, the optimization of submodular functions over  the \emph{integer lattice} (integer submodular functions) has recently attracted much interest, because this domain relates naturally to many practical problem settings, such as multilabel graph cut,  budget allocation and revenue maximization with discrete assignments. In contrast, the use of these functions for probabilistic modeling has received surprisingly little attention so far. 
In this work, we firstly propose the Generalized Multilinear Extension, a continuous DR-submodular extension for integer submodular functions. We study central properties of this extension and formulate a new probabilistic model which is defined through \emph{integer} submodular functions. Then, we introduce a block-coordinate ascent algorithm to perform approximate inference for those class of models. Finally, we demonstrate its effectiveness and viability on several real-world social connection graph datasets with integer submodular objectives.
\end{abstract}

\section{Introduction}

\looseness -1 Submodular functions have many applications in machine learning, which include data summarization~\citep{tschiatschek2014learning, Lin:2012:LMS:3020652.3020704}, sensor placement~\citep{sensorplacement} and computer vision~\citep{boykovfast}. They are defined on the subsets of the ground set $\V$ which contains $n$ elements. Alternatively, they can be described as a pseudo-Boolean function which is defined on vertices of the unit hypercube $\{0,1\}^n$. 
A crucial benefit of submodular functions is their attractive properties from the optimization perspective. Using the celebrated greedy algorithm of~\citet{Nemhauser1978}, one can get a constant factor approximation guarantee for maximizing a submodular function subject to  cardinality constraints. However, for more complex constraints, such as matroid constraints, the best guarantees can be obtained via a continuous relaxation of the problem. The {\em multilinear extension}~\citep{calinescu2007maximizing} is a continuous extension of a submodular function to the full hypercube $[0,1]^n$ and one can get the best approximation guarantee for matroid constraints by optimizing it using the continuous greedy algorithm \citep{Calinescu2011}.

\looseness -1 Recently, {\em Probabilistic Submodular Models (PSMs)} ~\citep{djolonga2014map,flid,gotovos2015sampling} were introduced as Gibbs distributions over subsets parameterized via a submodular energy function. These distributions can represent expressive probabilistic models by capturing complex interactions among elements. Since the partition function of these models requires a sum over all subsets, it is {\em intractable} for large ground sets. One way to approximate this probability distribution is to use {\em variational inference}. When a fully factorized Bernoulli distribution is chosen as a variational distribution, the lower bound for the log-partition function -- the {\em Evidence Lower Bound (ELBO)} -- can be decomposed into two parts: The multilinear extension and the entropy of the variational  distribution~\citep{bianmf}.

\looseness -1 Though originally associated with set functions, submodularity can be generalized to {\em integer} \citep{soma2014optimal,soma2018maximizing} and {\em continuous} domains \citep{Bach2019,staib2017robust,bian2017guaranteed,hassani2017gradient}. Submodular functions over the integer lattice have applications in optimal budget allocation \citep{soma2014optimal}, sensor placement with discrete energy levels \citep{DBLP:conf/nips/SomaY15}  and influence maximization with partial incentives \citep{demaine2014influence}. Submodular optimization over continuous domains captures a rich class of tractable non-convex optimization problems, and arises in various applications, such as revenue maximization with continuous assignments \citep{bian2017guaranteed}, 
robust budget allocation \citep{staib2017robust}
and provable mean-field inference for PSMs \citep{bianmf}. However, the family of {\em probabilistic submodular models defined through integer and continuous submodular functions} has been little explored, and it is, for example, not clear how to apply variational inference in a principled way. 

Our main contributions are summarized as follows:
\begin{enumerate}

    \item We present the \emph{Generalized Multilinear Extension} (GME) for integer  submodular functions and continuous submodular functions. The relation of GME to other continuous extensions is described in \cref{fig:contribution}. 
    Though being non-convex, GME is DR-submodular in general, which allows to use guaranteed DR-submodular maximization  techniques to conduct provable variational inference.
    
    \item We introduce {\em Probabilistic Integer Submodular Models}, which are characterized through integer submodular functions. We show that using an appropriate variational distribution we can obtain a DR-submodular ELBO
    for variational inference. This ELBO is obtained as an expectation w.r.t. the same variational distribution through which the GME is defined.

    \item To optimize this DR-submodular ELBO, we develop an efficient block coordinate ascent algorithm and show its effectiveness on real-world graph datasets.

\end{enumerate}
\begin{figure}[h]\centering
    \tabcolsep=0.11cm
    \begin{tabular}{l|l|c|c|}
      \multicolumn{2}{c}{}&\multicolumn{2}{c}{Optimization}\\
      \cline{3-4}
      \multicolumn{2}{c|}{}&Minimize & Maximize\\
      \cline{2-4}
      \multirow{4}{*}{\rotatebox[origin=c]{90}{Domain}}& \multirow{2}{*}{Set} & Lovasz Ext.  & Multilinear Ext. \\
      
                          & & \citep{lovasz1983submodular} & \citep{calinescu2007maximizing} \\
          \cline{2-4}
                          & \multirow{2}{*}{Integer} & Extended Ext. & Generalized ME \\
          
                          &  & \citep{Bach2019} & \textbf{This work} \\
      \cline{2-4}
    \end{tabular}
        \caption{Continuous extensions for submodular functions depending on the function domain and optimization type.}
\label{fig:contribution}
\end{figure}


\section{Problem Setting}
\looseness -1 We consider submodular functions defined on the subsets of $\R^n$ of the form $\X = \prod_{i=1}^n \X_i$, where each $\X_i$ is a compact subset of $\R$~\citep{Bach2019}. A function $f$ is \emph{submodular} if $\forall \x, \y \in \X$ we have
\begin{equation}
    f(\x) + f(\y) \geq f(\x \lor \y) + f(\x \land \y),
\end{equation}
where $\lor$ and $\land$ are component-wise maximum and minimum functions respectively. From now on, a submodular function $f$ will be referred as 
\begin{itemize}
    \item \emph{set submodular} if $\X_i = \{0,1\}$
    \item \emph{integer submodular} if $\X_i = \{0,1, \ldots, k-1\}$
    \item \emph{continuous submodular} if $\X_i = [a,b]$
\end{itemize}

\looseness -1 For ease of notation, we assume that $\X_i$ is same along each dimension, i.e.\ $k$ does not depend on $i$. Our proofs will work with different $\X_i$ as well, i.e.\ using $k_i$ instead of $k$. When $f$ is continuous submodular and twice differentiable, we have
\begin{equation}
    \forall i \neq j, \forall \x \in \X, \frac{\partial^2 f(\x)}{\partial x_i \partial x_j} \leq 0.
\end{equation}
\looseness -1 A certain {\em diminishing returns} property plays a  significant role for the maximization of set submodular functions. A function $f$ has diminishing return property if $\forall A \subseteq B \subseteq \V$, $e \notin B$ we have $f(A + e) - f(A) \geq f(B +  e) -f (B)$. For set functions, this diminishing return property is equivalent to submodularity. However, for integer functions, submodularity is a weaker condition than the diminishing returns property. An integer function is called \emph{diminishing return submodular} (\emph{DR-submodular}) if $\forall \x \leq \y \in \X$ and $i \in \V$, we have $f(\x + \eb_i) - f(\x) \geq f(\y + \eb_i) -f(\y)$ where $\eb_i$ is the $i$-th unit vector. When $f$ is continuous submodular and twice differentiable, DR-submodularity is equivalent to
\begin{equation}
    \forall i,j, \forall \x \in \X, \frac{\partial^2 f(\x)}{\partial x_i \partial x_j} \leq 0.
\end{equation}
Thus, submodularity only imposes constraints on {\em interactions} of coordinates, whereas DR-submodularity also requires that $f$ is {\em concave along each coordinate} (which of course does not require that $f$ is globally concave).

\looseness -1 We will now discuss an equivalence between the 
integer vectors and classical \emph{multisets}. With this 
connection, we have the following benefits: {\em i)} By borrowing well-established concepts, such as cardinality, intersection and set difference from multiset theory, we can define the corresponding concepts for integer vectors in a principled way. {\em ii)} The notions from multiset theory will ease the proof in the following part.  

\looseness -1 We follow the exposition from \citet{multiset1} and \citet{multiset2} to define basic notions of multisets.  
A $\emph{multiset}$ is a natural generalization of a set, where elements can be contained {\em repeatedly}, i.e., it may contain a finite number of copies of a particular element. In classical sets, in contrast, distinct elements  can only occur once. The number of times an element occurs in a multiset is called the \emph{multiplicity} $\mu(i)$ of the element $i$.

\looseness -1 Formally, a multiset $\MV$ is defined as a pair $\langle \groundset, \mu \rangle$, where $\groundset$ is the support and $\mu : \groundset \rightarrow \N $ is a function defining multiplicity for each element. Any ordinary set $\groundset$ can be viewed as a multiset $\langle \groundset,  \chi_{\groundset} \rangle$ where $\chi_{\groundset}$ is the characteristic function of the set $\groundset$.  A \emph{submultiset}  is a generalization of a subset. For each element in the support, you can choose a multiplicity to create a submultiset. Finally, we define the \emph{powerset} of a multiset $\PMV$ as the set of submultisets of $\MV$ and the cardinality of the powerset is defined by 
\begin{equation*}
    | \PMV | = \prod\nolimits_{i \in \groundset} (\mu(i) + 1).
\end{equation*}
\looseness -1  Note that this concept is consistent with the ordinary set theory where the multiplicity for each element is 1, i.e. $\mu(i)=1$ and we recover $|\mathcal{P}(\groundset)| = 2^n$.

\looseness -1 Given the definitions of multisets, it is clear that one can represent any multiset with an integer vector, by viewing the multiplicity of the multiset $\mu$ as the argument $\x$ of the integer vectors. 
With this connection, we can transfer several important notions from multisets to integer vectors, such as the notion of a subset, set intersection, set union and set difference. 
%

\looseness -1 Now we  define  the notion of the subset of a multiset and the operations between multisets. Suppose $\mathcal{A} = \langle \groundset, \mu \rangle$ and $\mathcal{B} = \langle \groundset, \nu \rangle$ are two multisets and the corresponding integer vectors are $\ba$ and $\bb$. Then $\mathcal{A}$ is a \emph{sub-multiset} of $\mathcal{B}$, $\mathcal{A} \subseteq \mathcal{B}$ if for all $e \in \groundset$ we have $\mu(e) \leq \nu(e)$ (or if $\ba\leq \bb$ for integer vectors). The \emph{set difference} of $\MA$ from $\MB$, $\MA \setminus \MB$, is the multiset $\mathcal{C} = \langle \groundset, \xi\rangle$
where for all $e \in \groundset$ we have $\xi(e) = \max(\mu(e) - \nu(e),0)$ (or $(\ba -\bb)\vee 0$ for integer vectors). 
The union and intersection of two multisets are similarly defined where $\xi(e) = \max(\mu(e), \nu(e))$ and $\xi(e) = \min(\mu(e), \nu(e))$ respectively (or $\ba \vee \bb$ and $\ba \wedge \bb$ for integer vectors).


\section{The Generalized Multilinear Extension}
\looseness -1 Let $f$ be an integer submodular function of $n$ variables and $\rhob_i\in \R_+^{k-1}$ be the marginals of  a $k$ dimensional categorical distribution. Let $\rhob :=  [\rhob_1; ...; \rhob_n]\in \R_+^{n \times (k-1)}$.
\looseness -1 The generalized multilinear extension $F$ is defined on the space of product of categorical distributions and can be written as:
\begin{equation}\label{eq_GME}
    F(\rhob_1, \ldots, \rhob_n)= \EV_{R(\rhob) \sim   \rhob_1, \ldots, \rhob_n}  [f(R(\rhob))].
\end{equation}
Hereby $R(\rhob)$ is a random integer vector where each dimension is sampled from an independent categorical distribution. In other words, each $\rhob_i$ lives in the $k-1$ dimensional simplex $\Delta^{k-1}$ and $\rhob$ is the concatenation of all $\rhob_i$ vectors. We define the simplex $\Delta^{k-1}$ as 
\begin{align*}
    \Delta^{k-1}:= \{  \rhob_i \in \R^{k-1} : \rho_{i,1} + \ldots + \rho_{i,k-1} \leq 1, \\
    \rho_{ij} \geq 0, j=1, \ldots,  k-1 \}.  
\end{align*} 
We define the union of $n$ simplexes as $\Delta^{k-1}_n$, and naturally $\rhob \in \Delta^{k-1}_n$. We need to compute the sum of $k^n$ elements to compute the expectation in \cref{eq_GME}. Note that when $k=2$, this extension corresponds to the multilinear extension of a set function. Here is an example with $n=2$ and $k=3$. In this case we have two categorical distributions which take three different values. 
\begin{align*}
   & F([ \rhob_1; \rhob_2]) = F(\rho_{11}, \rho_{12}, \rho_{21}, \rho_{22})= \\
   &f(0,0) (1 - \rho_{11} - \rho_{12}) (1 - \rho_{21} - \rho_{22}) + f(2,2) \rho_{12} \rho_{22} \\
    &f(1,0) \rho_{11} (1 - \rho_{21} - \rho_{22}) + f(0,1) (1-\rho_{11} - \rho_{12}) \rho_{21} + \\
    & f(2,0) \rho_{12} (1- \rho_{21} - \rho_{22}) + f(0,2) (1-\rho_{11} - \rho_{12}) \rho_{22} + \\
    & f(1,1) \rho_{11} \rho_{21} + f(1,2) \rho_{11} \rho_{22} + f(2,1) \rho_{12} \rho_{21},
\end{align*}
where we have the following constraints
\begin{align*}
    \rho_{11} + \rho_{12} \leq 1, \ \rho_{21} + \rho_{22} \leq 1, \  \rho_{ij} \geq 0, \  i,j=1,2.
\end{align*}
Therefore, maximizing the GME requires solving a constrained optimization problem with linear constraints. Now, we are going to answer the following questions:

\begin{enumerate}
    \item \textit{What are the properties of this GME when the integer function is submodular?}
    \item \textit{Does it have similar properties to the multilinear extension which makes it suitable for  optimization purposes?}
\end{enumerate}

This generalized multilinear extension in \cref{eq_GME} has the following properties:
\begin{proposition}\label{prop_gme}
Let $F$ be the generalized multilinear extension of an integer submodular function $f$. Then we have
\begin{enumerate}
    \item If $f$ is monotone then $\frac{\partial F }{\partial \rho_{ij}} \geq 0$ for all $i\in \V$ and $j \in \{1, \ldots, k-1\}$;
    
    \item F is DR-submodular (even if $f$ is not), i.e.,  $\frac{\partial^2 F }{\partial \rho_{ij} \partial \rho_{kl}} \leq 0$ for all $i,j,k,l$.
\end{enumerate}
\end{proposition}

\begin{proof}
\looseness -1 Given $\rhob \in \Delta^{k-1}_n$, let $\Rb$ be a random multiset where elements appear independently with probabilities $\rhob_i$. Since $F$ is multilinear, the partial derivative can be written as a difference of two GME.
\begin{align*}
    \frac{\partial F}{\partial \rho_{ij}} &= F(\rhob_1, \rhob_i = \eb_j, \rhob_n) - F(\rhob_1, \rhob_i = \mathbf{0}, \rhob_n) \\
    & = \EV_{\Rb \sim \rhob}[f(\Rb + \mathcal{E}_{i}^{j} )] - \EV_{\Rb \sim \rhob}[f(\Rb - \mathcal{E}_{i}^{j})]\\
    & \geq 0,
\end{align*}
since f is monotone and $\Rb - \mathcal{E}_{i}^{j} \subseteq \Rb + \mathcal{E}_{i}^{j} $ where $\mathcal{E}_{i}^{j}$ is the multiset where $i$th element has multiplicity $j$.

For the second statement, we first observe that $\frac{\partial^2 F }{\partial \rho_{ij} \partial \rho_{kl}} = 0$ when $i=k$ for all $j,l$ since $F$ is multilinear. 
\begin{align*}
    \frac{\partial^2 F }{\partial \rho_{ij} \partial \rho_{kl}} &= \frac{\partial}{\partial \rho_{ij}} \left( F(\rhob_k= \eb_l, \rhob_n) - F(\rhob_k= \mathbf{0}, \rhob_n) \right) \\
    &= F(\rhob_i = \eb_j, \rhob_k=\eb_l) + F(\rhob_i=\mathbf{0},\rhob_k=\mathbf{0}) \\
    &-F(\rhob_i=\eb_j, \rhob_k=\mathbf{0}) - F(\rhob_i=\mathbf{0}, \rhob_k=\eb_l) \\
    &= \EV[f(\Rb + \mathcal{E}_{i}^{j} + \mathcal{E}_{k}^{l})] + \EV[f(\Rb - \mathcal{E}_{i}^{j} - \mathcal{E}_{k}^{l})] \\
    &- \EV[f(\Rb + \mathcal{E}_{i}^{j} - \mathcal{E}_{k}^{l})] - f(\Rb - \mathcal{E}_{i}^{j} + \mathcal{E}_{k}^{l}) \\
    & \leq 0,
\end{align*}
since $f(\Rb + \mathcal{E}_{i}^{j} + \mathcal{E}_{k}^{l}) +  f(\Rb - \mathcal{E}_{i}^{j} - \mathcal{E}_{k}^{l}) \leq f(\Rb + \mathcal{E}_{i}^{j} - \mathcal{E}_{k}^{l}) + f(\Rb - \mathcal{E}_{i}^{j} + \mathcal{E}_{k}^{l})$ because of the integer submodularity of $f$. 
\end{proof}

Some remarks about \cref{prop_gme} are in order:
\begin{enumerate}
    \item  We \emph{do not} require $f$ to be  DR-submodular  in order to get a DR-submodular continuous extension. Even if $f$ is only integer submodular we still get a DR-submodular extension.

    \item This generalized multilinear extension can also be applied when $f$ is continuous submodular: By discretizing the interval $[a,b]$ where the function is defined, firstly we get an integer submodular function. Then applying the generalized multilinear extension, we get a DR-submodular extension.
\end{enumerate}

\subsection{On the Relation with Other Continuous Extensions}
To our knowledge, there are two other continuous extensions \citep{soma2018maximizing,Bach2019} that can be used for integer submodular functions. 

\looseness -1 \citet{soma2018maximizing} define a continuous extension in order to maximize an integer submodular function subject to a polymatroid constraint. For $\x \in \R^\groundset$, let $\floor{\x}$ denote the vector obtained by rounding down each entry of $\x$. For $a \in \R$, let $\innerprod{a}$ denote the fractional part of $a$, namely $\innerprod{a} = a - \floor{a}$. For $\x \in \R$, $\mathcal{D}(\x)$ is defined as the distribution from which $\overline{\x}$ is sampled such that $\overline{\x}(i) = \floor{\x(i)}$ with probability $1 - \innerprod{\x(i)}$ and $\overline{\x}(i) = \ceil{\x(i)}$ with probability $\innerprod{x(i)}$ for each $i \in \groundset$. Then the continuous extension $F_{S} :\R^{\groundset}_{+} \rightarrow \R_{+}$ is defined as
\begin{align*}
    F_{S}(\x) & = \EV_{\overline{\x} \sim \mathcal{D}(\x)}[f(\overline{\x})] \\
    &= \sum_{S \subseteq \groundset} f(\floor{\x} + \eb_{S}) \prod_{i \in S} \innerprod{\x(i)} \prod_{i \notin S} (1 - \innerprod{\x(i)}).
\end{align*}
In this case, $F$ is obtained by gluing the multilinear extension of $f$ restricted to each hypercube. However, this extension is non-smooth and it is not clear how it can be used in variational inference for PSMs with integer submodular functions. On the other hand, the proposed GME is smooth and corresponds to the one part of the ELBO for variational inference for PSMs with integer submodular functions as we show below. 

Another continuous extension, which is suited for submodular {\em minimization} is presented in \citet{Bach2019}. It is originally presented for continuous submodular functions but can also be used for integer submodular functions by discretizing the continuous interval.  The extension $f_{\text{cumulative}}$ is defined on the 
set of products of Radon probability measures $\mu \in \prod_{i=1}^n P(\mathcal{X}_{i})$ as:
\begin{align*}
    f_{\textrm{cumulative}}(\mu_1, \ldots, \mu_n) = \int_{0}^{1} f(\theta_1(t), \ldots, \theta_n(t)) dt,
\end{align*}
where $\theta_i(t) = \sup \{ x_i \in \mathcal{X}_i, F_{\mu_i}(x_i) \geq t\}$ with $F_{\mu_i}(x_i)$ denoting the cumulative distribution function of $\mu_i$. This extension is convex when the function $f$ is submodular. 

Our result is consistent with  previous reduction techniques for DR-submodular functions~\citep{ene_reduction} which depend logarithmically on $k$.
In their setting, they first apply their reduction to get a ground set of size $n \log k$, then use the classical multilinear extension which requires the sum over $2^{n \log k } = k^n$ values. On the other hand, our GME is based on the ground set of size $nk$ but requires the sum over $k^n$ values. If the classical multilinear extension was used, $2^{kn}$ values would be summed up. In short, our GME requires the same number of function evaluation as the classical multilinear extension applied to the reduced ground set.


\section{Variational Inference for  Probabilistic Integer Submodular Models}\label{categorical}

Now, we consider distributions over multisets of $\groundset$, namely $p(\x) = \tfrac{1}{Z} e^{f(\x)}$, where $\x$ is defined on the integer lattice $\x \in \{0, \ldots, k-1\}^\groundset$. We call such distributions \emph{Probabilistic Integer Submodular Models} (PISMs) and they  generalize the existing Probabilistic Submodular Models which are defined on the vertices of the unit hypercube. PISMs capture various models in practice. For example, the repulsive Potts model with  discrete labels \citep{nieuwenhuis2013survey}, and repulsive Markov random fields with categorical random variables \citep{li2009markov} are such examples. 

\looseness -1 In our case, w.l.o.g, we assume that each element of $\groundset$ has $k$ levels, therefore the corresponding multiset is defined as $\{ \underbrace{e_1, \ldots, e_1}_{k-1 \ \textrm{times}}, \ldots, \underbrace{e_{\groundset}, \ldots, e_{\groundset}}_{k-1 \ \textrm{times}}\}$. Since not having an element is also defined as a level, the multiplicity of each element is $k-1$ with $|\PMV| = k^n$. The normalization constant 
\begin{equation}
    Z = \sum_{\x \in \PMV} e^{f(\x)}
\end{equation}
is also called the partition function and involves a summation over $k^n$ terms. When $k$ and $n$ are large, computing this sum becomes intractable and variational inference~\citep{blei2017variational} is one of the widely used methods to approximate the probability density $p(\x)$. Let $\q$ be a variational distribution with its variational parameter $\rhob$. When we minimize the KL divergence between $\q$ and $p(\x)$ we have
\begin{align*}
     & \kl{\q}{p(\x)} = \sum_{\x \in \PMV}  \q \log \frac{\q}{p(\x)}\\
     &= \sum_{\x \in \PMV} \q \log \q  \\
     &- \sum_{\x \in \PMV} \q f(\x)  + \log Z.
\end{align*}
Since KL divergence is always non-negative, we get an expression which gives a lower bound for the log partition function. This term is called {\em Evidence Lower Bound~(ELBO)} and can be written as:
\begin{align*}
    \text{ELBO}(\rhob) &= \sum_{\x \in \PMV } \q f(\x) \\
    &- \sum_{\x \in \PMV} \q \log \q \leq \log Z.
\end{align*}
Usually, the ELBO  is non-convex and difficult to optimize if we do not make additional assumptions. One of the most popular assumptions is the mean-field assumption which uses a fully-factorized variational distribution. If  assuming that our variational distribution belongs to the mean-field family, i.e., $\q = \prod_{i}q_i(x_i, \rhob_i)$, where each $q_i$ is a categorical distribution with parameter $\rhob_i$, the ELBO objective becomes a sum of the GME of $f$ and the entropy of categorical distributions. Moreover, the entropy part of ELBO becomes separable and we can write:
\begin{align*}
    \text{ELBO}(\rhob) = F(\rhob) + \sum_{i \in \groundset} H(\rhob_i),
\end{align*}
where $H$ is the entropy of one dimensional categorical distribution. 
The following proposition establishes the DR-submodularity of $\text{ELBO}(\rhob)$, all the omitted proofs in the sequel are deferred to \cref{appendix_proofs}.
\begin{proposition}
$\text{ELBO}(\rhob)$ is DR-submodular in $\rhob$ as long as the energy function $f$ is integer submodular.
\end{proposition}
Since ELBO is DR-submodular we can use various DR-submodular maximization algorithms \citep{mokhtari2017conditional,hassani2017gradient,biannips2017nonmonotone} to optimize it.


\section{Extension to Continuous Variables}

\looseness -1 It is also relevant to consider {\em continuous}
PSMs, defined using a \emph{continuous DR-submodular} function $f(\x)$ where the domain is the hypercube $[a,b]^n$. Here, the partition function becomes $Z = \int \exp(f(\x)) d\x$. For large values of $n$, this integral becomes intractable and we will use variational inference to approximate $p(\x)$. Let $q(\x,\thetab)$ be a variational distribution with its variational parameter $\thetab$. Following the same logic in the previous section, the ELBO becomes
\begin{align*}
    \text{ELBO}(\thetab) =  \underbrace {\Int q(\x,\thetab) f(\x) d\x}_{L(\thetab)} - \Int q(\x,\thetab) \log q(\x,\thetab) d\x.
\end{align*}
We again make a mean-field assumption about the variational distribution $q(\x,\thetab)$.  $\text{ELBO}(\thetab)$ is still highly non-convex in this circumstance.  In order to provably maximize the $\text{ELBO}(\thetab)$, we ideally would want it to have nice structure, such as DR-submodularity w.r.t. $\thetab$. 
This amounts to choosing some family of variational distributions $q(\x,\thetab)$ such that $\text{ELBO}(\thetab)$ is still tractable. 
Then the following important  question arises:

\textit{For which variational family is the ELBO objective  DR-submodular w.r.t. to the variational parameter $\thetab$?}

We consider firstly the first part of the ELBO ($L(\thetab)$) and give a sufficient condition such that it   becomes DR-submodular in \cref{ctsELBO}. Because of the mean-field assumption, for \emph{simple} variational distributions, the entropy part becomes separable and concave. Therefore, we obtain a DR-submodular ELBO.
\begin{proposition}\label{ctsELBO}
If $q(x, \theta)$ is a probability density function such that $q_{x}''(x, \theta) = q_{\theta}''(x, \theta)$, then we have
\begin{equation}
    L''(\theta) = \Int_{-\infty}^{+\infty} q_{\theta}''(x, \theta) f(x) dx = \Int_{-\infty}^{+\infty} q(x, \theta) f''(x) dx
\end{equation}
for sufficiently smooth $f$. This also holds in higher dimensions under the mean-field assumption.
\end{proposition}
Now, we look at various probability distributions and for distributions such as Gaussian and Laplace, this property holds. For example, if we choose a Gaussian with fixed variance, where the variational parameter is $\mu$, we have $q(x, \mu) = \tfrac{1}{\sqrt{2 \pi \sigma^2}} \exp(-\tfrac{1}{2\sigma^2} (x-\mu)^2 )$ with $\tfrac{\partial^2 q }{\partial \mu^2} = \tfrac{\partial^2 q }{\partial x^2}$ which satisfies the above proposition. This implies that first part of the ELBO is DR-submodular. Since the differentiable entropy part is $\sum_{i \in \groundset} \log (\sigma_i \sqrt{2 \pi e})$ and does not depend on $\mu$, ELBO is also DR-submodular with respect to the variational parameter $\mu$. In this case, since there is no entropy regularizer, ELBO puts all the mass on the maximizer of $f$, which is not desirable. 

One way to solve this issue is to discretize the interval and use the categorical distribution similar as that in Section~\ref{categorical}. Then the ELBO objective becomes DR-submodular. This discretization approach is also used in \citet{Bach2019} to minimize the convex extension of submodular functions.


\section{Algorithms for Approximate Mean-Field Inference}

Approximate inference amounts to solving the following non-convex optimization problem:
\begin{equation}\label{eq_elbo_objective}
    \max_{\rhob \in \Delta^{k-1}_n} \text{ELBO}(\rhob) = F(\rhob) + H(\rhob),
\end{equation}
\looseness -1 where $H(\rhob)$ is the entropy of the $n$ categorical distributions. The optimization problem is one instance of the constrained non-monotone DR-submodular maximization problem, for which several guaranteed algorithms exist. For example, the \shrunkenfw algorithm has a $1/e$ approximation guarantee and the \twophasefw algorithm has a $1/4$ approximation guarantee \citep{biannips2017nonmonotone}. In addition to these algorithms, we present a block coordinate ascent algorithm, abbreviated as  \blockca, which tries to optimize the marginals for one categorical  random variable in each iteration. 

\subsection{The Block Coordinate Ascent Algorithm}
We summarize the main steps in \cref{alg_blockca}, where  $\sete{\rhob^{k-1}}{i_k}{\xib_{}}$ means changing the marginals of the $i_k$-th random variable to be $\xib$ while keeping all other marginals fixed. \cref{alg_blockca} starts by initializing $\rhob^0$ to be some point in the probability simplex. Then it runs for $K$ iterations.
In  iteration $k$, it optimizes the marginals of random variable $i_k$ while keeping all other marginals fixed, which is illustrated in \cref{step_i_k}. 
For the ELBO objective used in this work, one can verify that it has a closed form solution, as shown by the following proposition.

\begin{algorithm}[ht]
	\caption{The \blockca Algorithm}\label{alg_blockca} 
	\KwIn{$\max_{\rhob \in \Delta^{k-1}_n}$ \text{ELBO}($\rhob$), \# iterations $K$
	}
	
	{Initialize $\rhob^0 \in \Delta^{k-1}_n$, $k \leftarrow 1$\;}
	
	\While{$k\leq K$}{
		{let $i_k$  be the random variable being operated\;}
		
		{

		$\xib_{i_k} \leftarrow  \argmax_{\xib}
	f(\sete{\rhob^{k-1}}{i_k}{\xib})$\;\label{step_i_k}}
	
		{$\rhob^k \leftarrow \sete{\rhob^{k-1}}{i_k}{\xib_{i_k}}$\;}

		{$k \leftarrow k + 1$\;}
	}
	\KwOut{$\rhob^{K}$ }
\end{algorithm}

\begin{proposition}
When maximizing $\text{ELBO}(\rhob)$ in \cref{eq_elbo_objective}, if one only optimizes for $\rhob_{i}$ while keeping all other marginals fixed, we have the following closed form solution (let $\nabla_{ij}:=\nabla_{\rho_{ij} }F(\rhob)  $ for notational simplicity): 

\begin{align}
\rho_{ij} = \frac{\exp{( \nabla_{ij} )} }{1 +    \sum_{j'}\exp{( \nabla_{ij'} )}}, \forall j \in \{1, ..., k-1\}.
\end{align}

\end{proposition}
\looseness -1 Since the Block Coordinate Ascent algorithm can only {\em increase} the ELBO in each iteration, if we initialize it using the solutions of \shrunkenfw and \twophasefw and therefore still retain their respective approximation guarantees.


\section{Experiments}

\looseness -1 We now illustrate the efficacy of the proposed generalized multilinear extension and algorithms on two classes of applications: Revenue maximization with discrete assignments and
facility location with discrete levels. 
Both of these two applications have integer submodular objectives.  
With a probabilistic treatment using the proposed PISMs, we are able to i) provide refined probabilistic solutions using variational inference techniques, instead of the discrete solutions without a probabilistic treatment; ii) conduct parameter learning through, for example,  the maximum likelihood framework, given that we have enough training data of the corresponding objectives. Due to limit of training data, we are only going to illustrate the first benefit in this section. 

We used the oblivious step sizes $2/(t+2)$ for the \twophasefw algorithm where $t$ is the iteration index. Below, by ``one epoch'' we mean one full gradient evaluation of the objective. 
Notice that in general it costs $\mathcal{O}(k^n)$ time 
to evaluate $F(\rhob)$ precisely. In practice, we use Monte Carlo sampling to estimate $F(\rhob)$ by sampling the integer function values according to the categorical distribution. Thanks to the Hoeffding bound \citep{hoeffding1963probability}, one can easily show that the estimated function value can be arbitrarily close to true $F(\rhob)$ with polynomial number of samples. 

\subsection{Revenue Maximization with Discrete Assignments}

\looseness -1 Given a social connection graph with nodes describing $n$ users and edges describing 
their connection strength, the viral marketing strategy suggests to  choose a small subset of buyers to give
them some product for free, to trigger a cascade of further adoptions
through ``word-of-mouth'' effects, in order to maximize the total revenue \citep{hartline2008optimal}. 
\looseness -1 One model with ``discrete'' product assignments was used by \cite{soma2017non} and \citet{durr2019non}, which is motivated by the fact that giving a user more free product, the user will be more likely to advocate the product. It can be viewed as  a simplified variant of the Influence-and-Exploit (IE) strategy of \citet{hartline2008optimal}. 

\begin{figure*}[htbp]

\begin{subfigure}[t]{0.36\textwidth}
    \includegraphics[width=\linewidth]{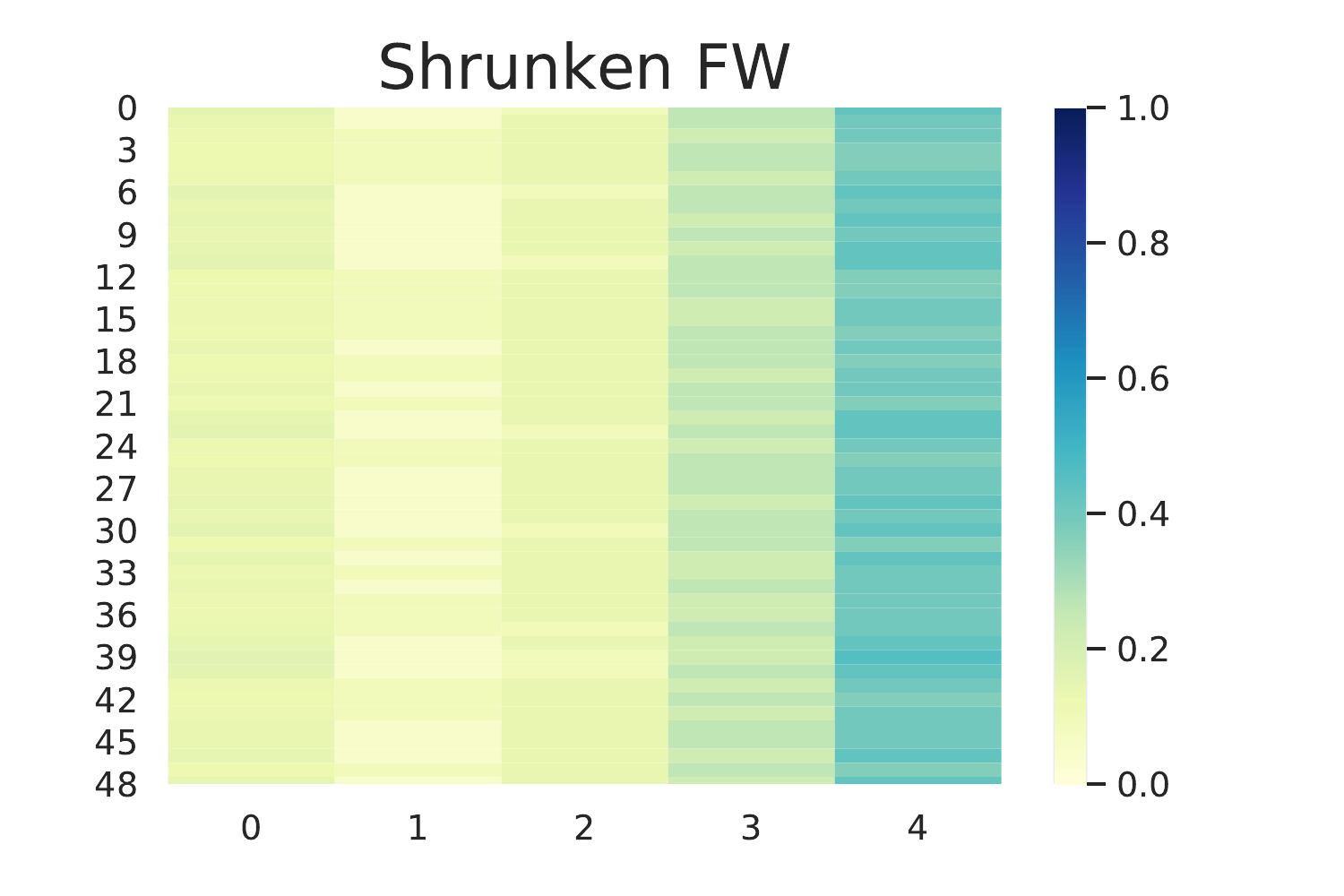}
\end{subfigure} \hspace{-10mm}
\begin{subfigure}[t]{0.36\textwidth}
  \includegraphics[width=\linewidth]{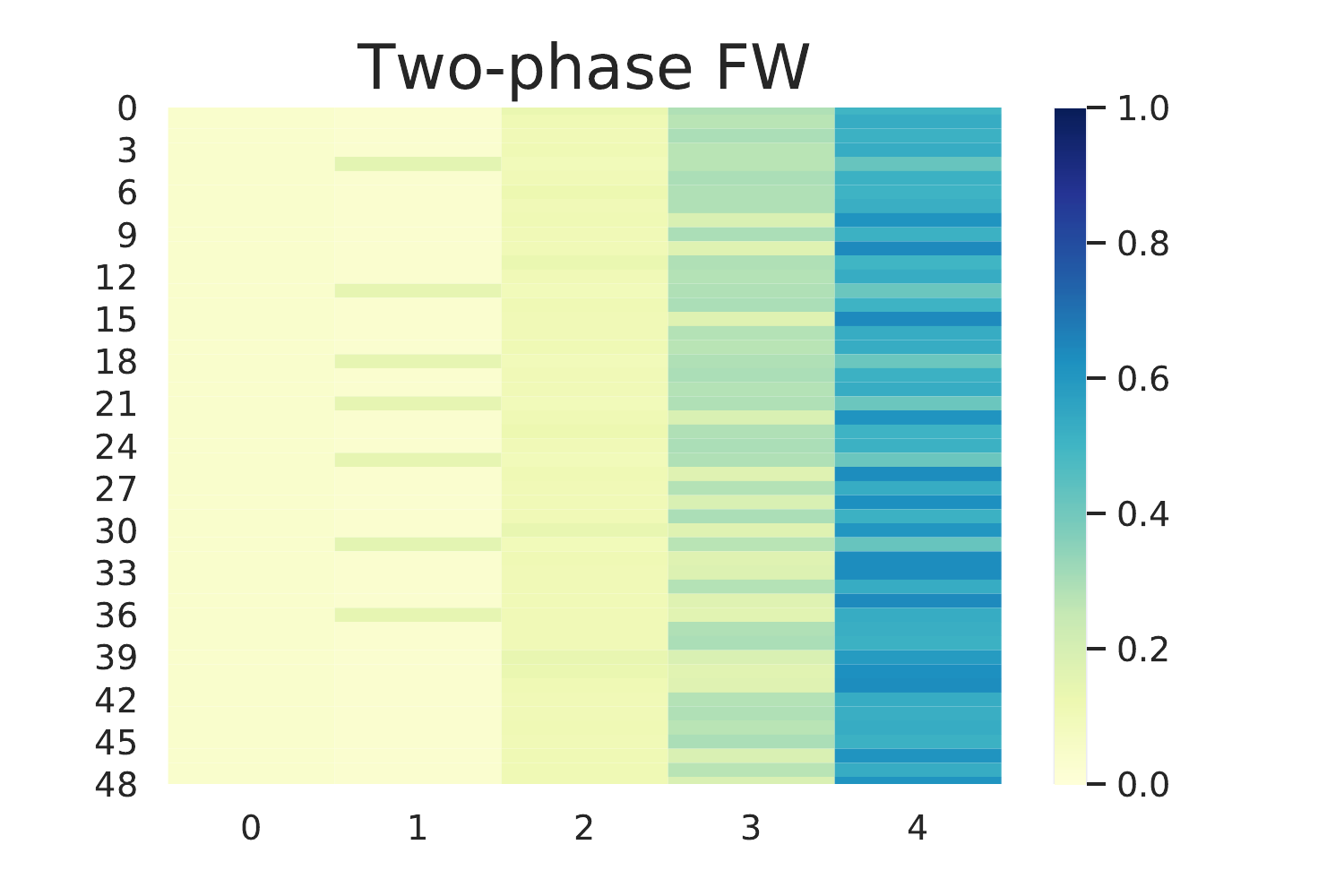}
\end{subfigure} \hspace{-10mm}
\begin{subfigure}[t]{0.36\textwidth}
    \includegraphics[width=\linewidth]{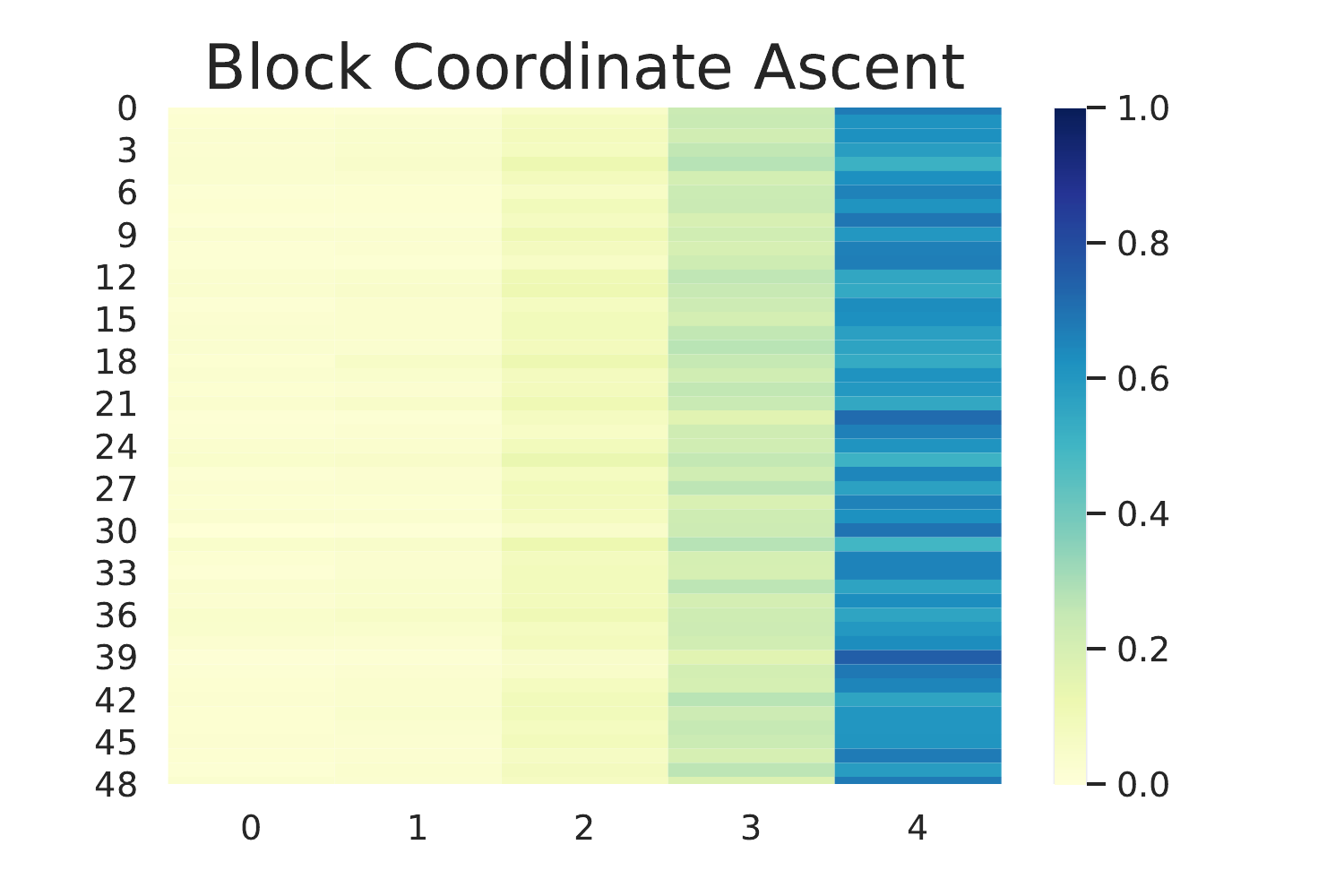}
\end{subfigure}
\begin{subfigure}[t]{0.36\textwidth}
    \includegraphics[width=\linewidth]{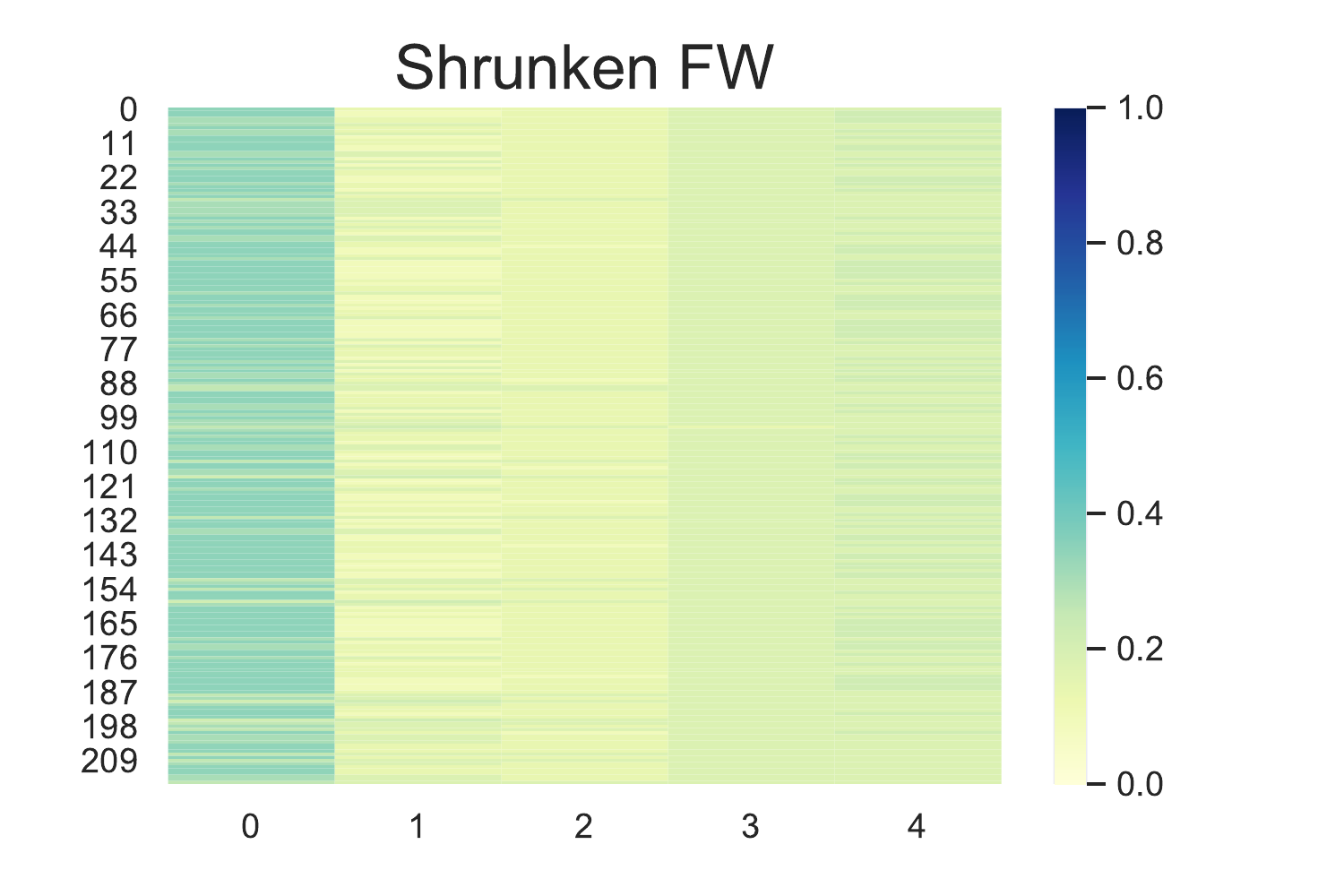}
\end{subfigure} \hspace{-8mm}
\begin{subfigure}[t]{0.36\textwidth}
  \includegraphics[width=\linewidth]{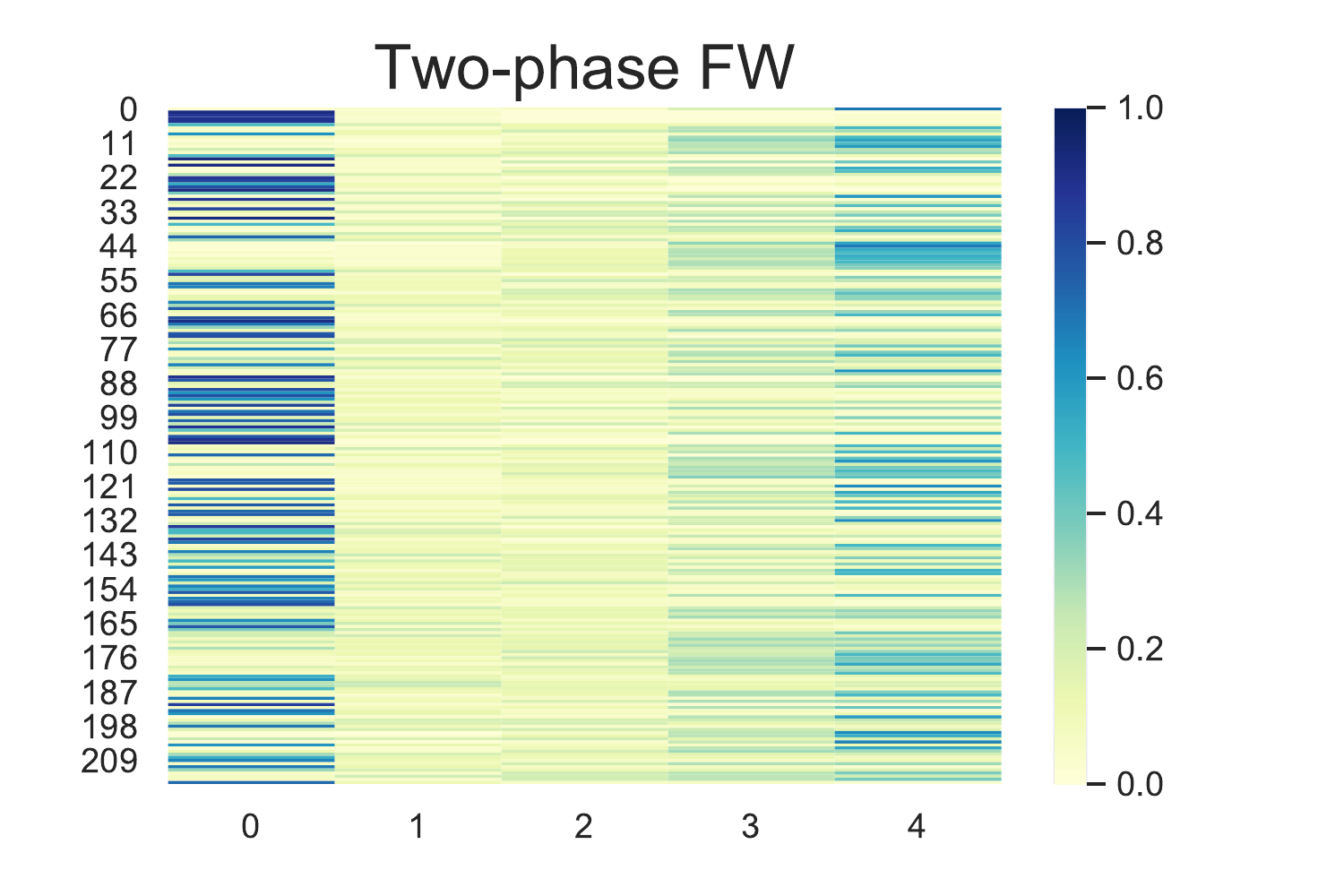}
\end{subfigure} \hspace{-8mm}
\begin{subfigure}[t]{0.36\textwidth}
    \includegraphics[width=\linewidth]{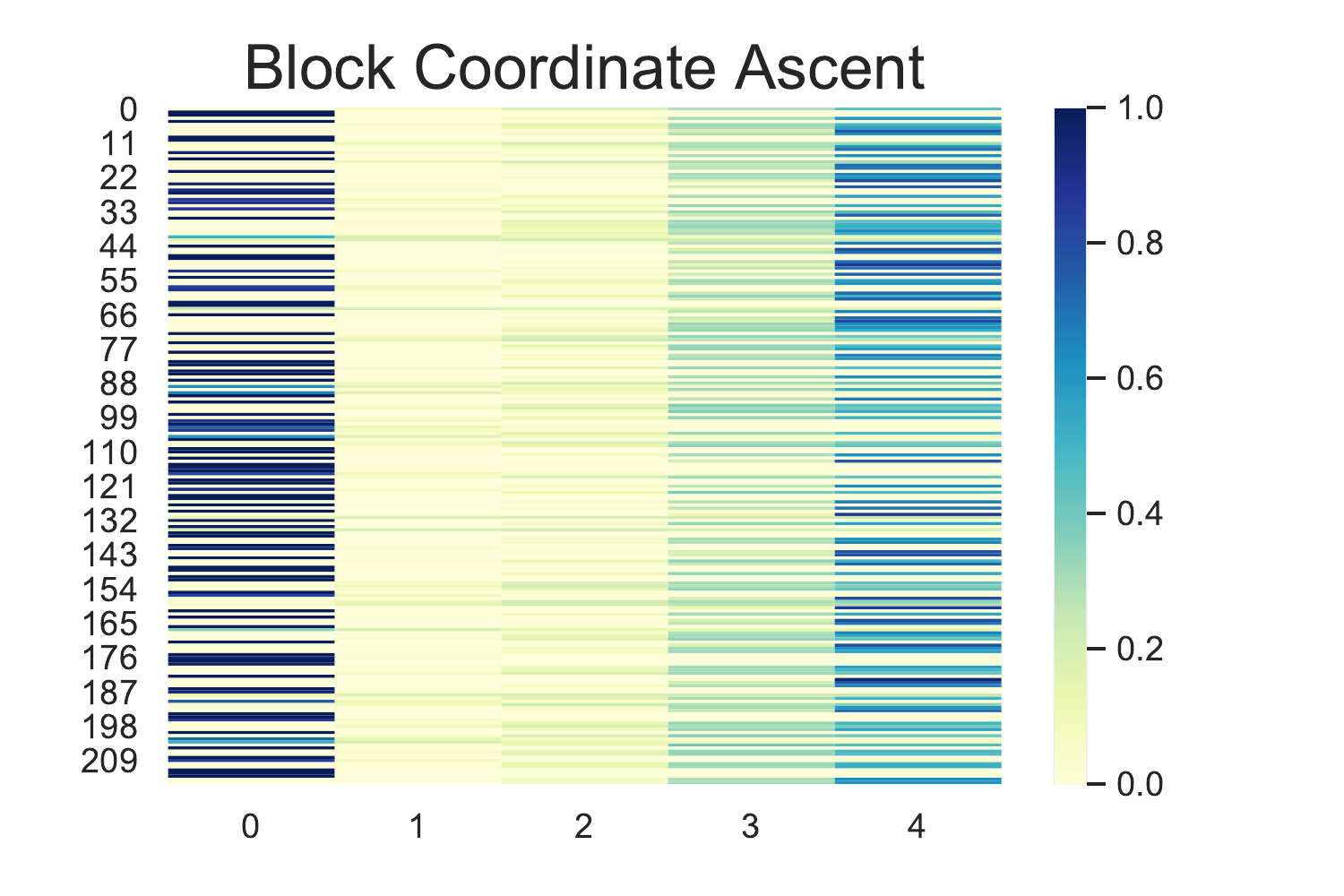}
\end{subfigure}

\looseness -1 \caption{Marginals of variational distribution for two different models: First row corresponds to the model with facility location function with 50 nodes and 5 levels. Second row corresponds to the model with revenue maximization function on the real world graph with 217 nodes and 5 levels. Each row corresponds to elements in the ground set and each column corresponds to their levels. When a box is darker, there is more probability mass assigned to that level. In both models, one can observe that the marginals of \shrunkenfw are smoother than the marginals of \blockca.}
\label{fig:marginals}
\end{figure*}

\looseness -1 In the \emph{influence} stage, for each of the user $i$, if giving him $x_i$ units of products for free, the user becomes an advocate of the product with probability $1 - q^{x_i}$ (independently from other users), where $q\in (0, 1)$ is a parameter.
In the \emph{exploit} stage: suppose that  a set $S$ of users advocate the product while the complement set $\groundset \setminus S$ of users do not. Now the revenue comes from the users in $\groundset \setminus S$, and they will be influenced by the advocates with probability proportional to 
the edge weights.
So the expected revenue is
a function $f: \Z_+^\groundset \rightarrow \R_+$: 
\begin{align}\notag 
	f(\x)  
	&= \epe[S]{\sum_{i\in S} \sum_{j\in \groundset\setminus S}W_{ij} } \\
	 & = \sum_{i\in \groundset} \sum_{j\in \groundset\setminus \{i\}} W_{ij} (1- q^{x_i})q^{x_j},
\end{align}
where $W$  is the adjacency matrix of the social connection graph. 

\begin{figure*}[!ht]
\begin{subfigure}[t]{0.34\textwidth}
    \includegraphics[width=\linewidth]{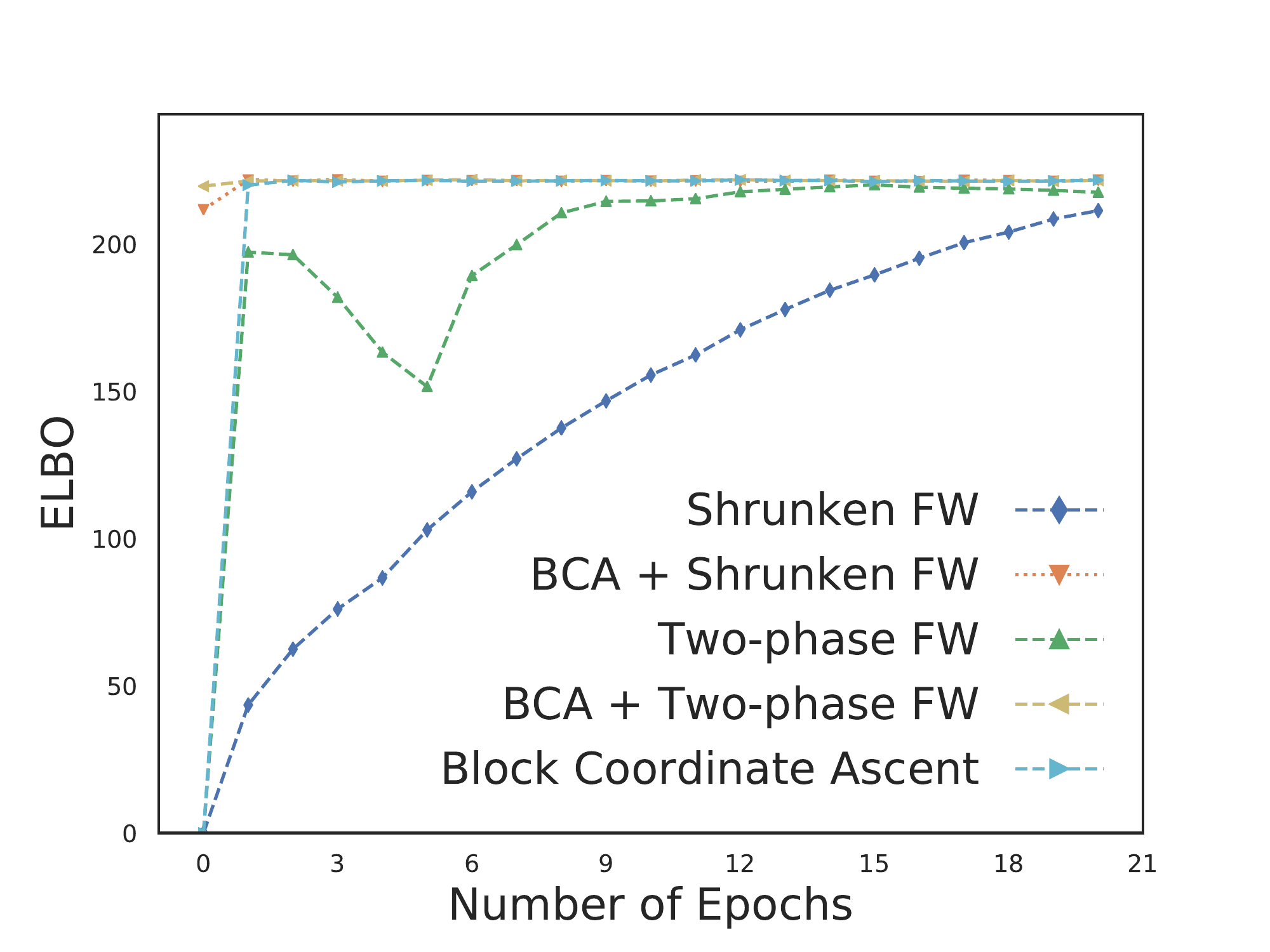}
\caption{Synthetic Dataset}
\label{fig:facility_location}
\end{subfigure} \hspace{-5mm}
\begin{subfigure}[t]{0.34\textwidth}
  \includegraphics[width=\linewidth]{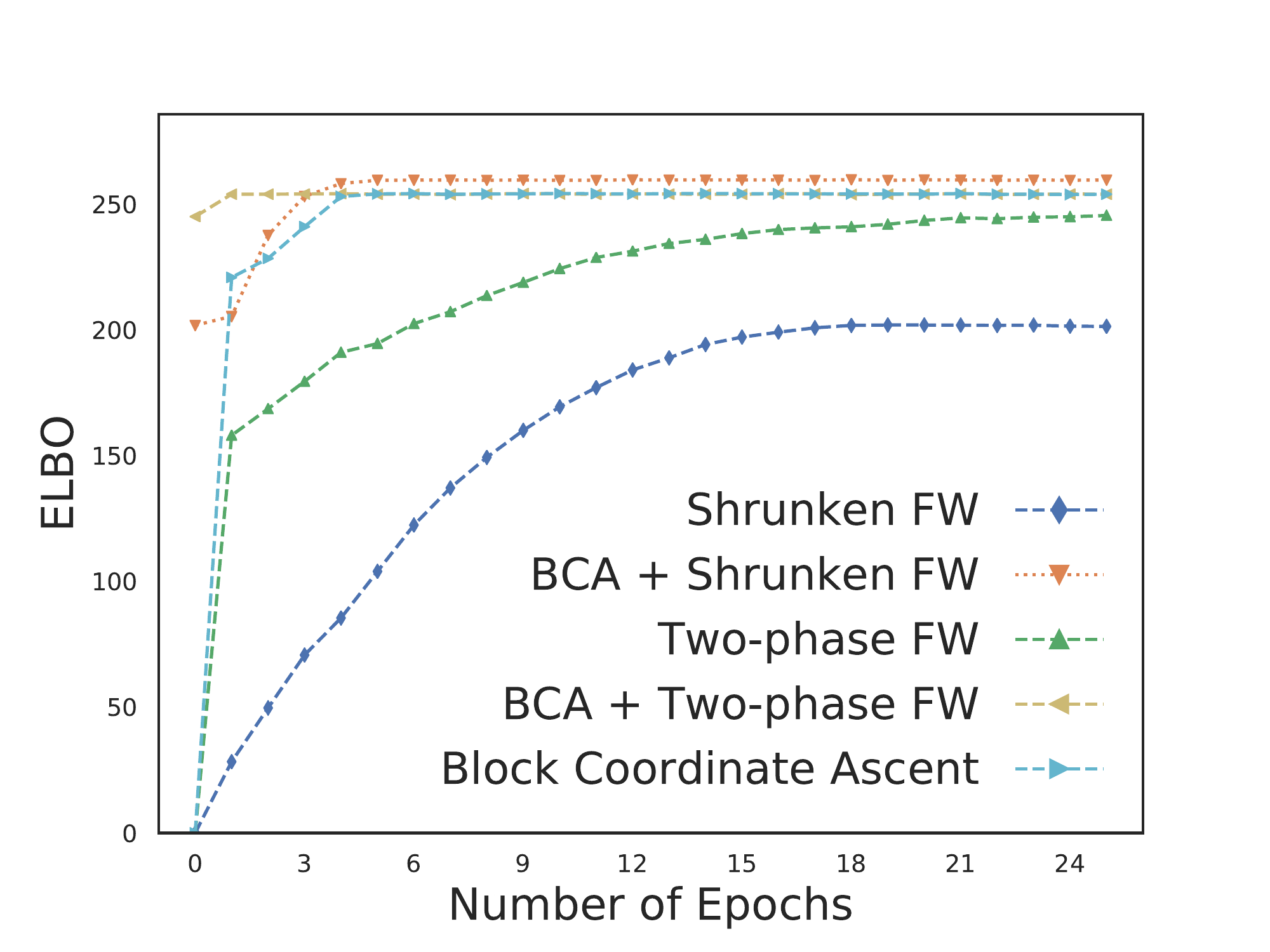}
\caption{``Football'' dataset}
\label{fig:football_trajectory}
\end{subfigure} \hspace{-5mm}
\begin{subfigure}[t]{0.34\textwidth}
    \includegraphics[width=\linewidth]{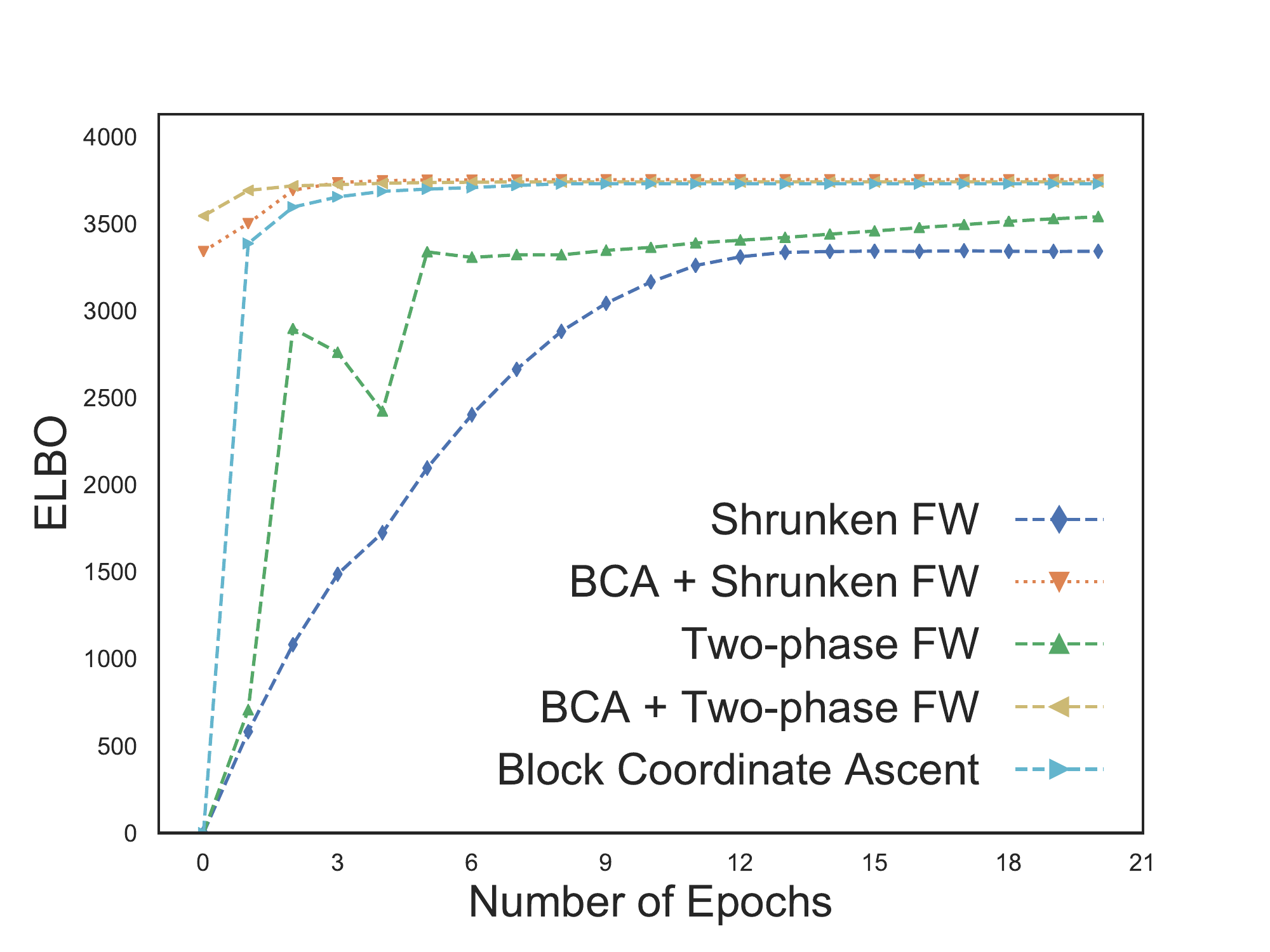}
\caption{``Residence Hall' dataset}
\label{fig:residence_trajectory}
\end{subfigure}
\caption{ELBO trajectories with different datasets. In all cases, \blockca achieves highest ELBO value and fastest convergence. \shrunkenfw is the slowest to converge and obtains the lowest ELBO value. \twophasefw obtains faster convergence than \shrunkenfw. }
\label{fig:trajectories}
\end{figure*}

\paragraph{Experimental results.}

\looseness -1 For variational inference of the PISMs which are defined through the above revenue objective, we used real-world graphs from the Konect network collection \citep{kunegis2013konect}\footnote{\url{http://konect.uni-koblenz.de/networks}}, for example,  the network of
``Residence Hall'' ($n$ = 217, \#edges = 2,672) and from Pajek dataset\footnote{\url{http://vlado.fmf.uni-lj.si/pub/networks/data/}}, for example, the network of ``Football'' ($n$ = 35, \#edges = 118).
Full information on the datasets and experimental results are deferred to \cref{appendix_experiments} due to space limit. 

\looseness -1 \cref{fig:football_trajectory} shows the trajectories of the  \shrunkenfw, \twophasefw and \blockca which is initialized by the output of the \shrunkenfw and \twophasefw for the ``Football'' dataset. One can see that \blockca type algorithms converge the fastest, \twophasefw follows, \shrunkenfw is  the last. \blockca which is initialized by the result of the \shrunkenfw attains the highest ELBO value, followed by \blockca initialized by \twophasefw. They both outperform vanilla \blockca initialized by $0$. \shrunkenfw and \twophasefw converge to the almost same value, both significantly lower than \blockca variants. The first line of \cref{fig:marginals} presents the marginals $\rhob$ of the variational distribution. We can see that \blockca tends to produce sharp marginals, while \shrunkenfw and \twophasefw provide smoother marginals. This is consistent with the intuition that Frank-Wolfe style algorithms usually generate solutions that lie in the interior of the constraints. 

\looseness -1 \subsection{Facility Location with Multiple Levels}
We use the facility location function with different levels of facilities from \citet{soma2014optimal}. In classical submodular facility location setting, we decide to either open a facility or not. Here, we have $|\groundset |$ facilities and we would like to decide on the level of the facility we want to open in order to serve $m$ number of customers. The level of the facilities is represented as integers in $\{0, \ldots, k-1 \}$. If a facility $j$ of level $c_j$ is opened, then it provides utility $w_{ij}(c_j)$ to the customer $i$. We choose $w_{ij}: \{0, \ldots, k-1\} \rightarrow \R$ as a monotone function and each customer chooses the facility with highest value. When there are multiple facilities with levels $c_j$, the total value provided to the customers is given by $\sum_{i=1}^m \max_{j \in \groundset} w_{ij}(c_j)$. This function is submodular and monotone over the integer lattice. 

\paragraph{Experimental results.}
\looseness -1 In this experiment, we use an integer facility location function with $n=50$ facilities and customers and $k=5$ levels. First, we generate matrices $L_i$ where each element follows a standard normal distribution for $i \in \{1,\ldots, k\}$. Then we compute the matrix $\tfrac{1}{n}L_i L_{i}^T$ and take their absolute values. For each level we sum up the previous utility values in order to get a monotone function. We run \shrunkenfw, \twophasefw and  variants of \blockca for $20$ epochs and the variants of \blockca achieve similar ELBO values whereas \shrunkenfw and \twophasefw obtain a lower ELBO value. In the second row of \cref{fig:marginals} we present a heatmap of the probabilities for the variational distribution $\rhob$. The $x$-axis represents different levels and the $y$-axis represents different facilities. If we look at the marginals of the \shrunkenfw, we see that there is a gradual increase in the probabilities assigned to different levels. On the other hand, \blockca assigns very high probability to the highest level and assigns lower probability to the remaining levels. This behavior can be related to the nature of the Frank-Wolfe style algorithms. Since \shrunkenfw uses convex combination of the vertices of the constraint set, the marginals are smoother than the \blockca.


\section{Related Work}

\looseness -1 The first systematic study for probabilistic models defined through submodular {\em set} functions is presented by \citet{djolonga2014map}. \citet{Djolonga:2016:VIM:3157096.3157294} develop an efficient approximate inference algorithm for more general models with both submodular and supermodular functions. \citet{gotovos2015sampling} analyze MCMC sampling to perform approximate inference in PSMs and \citet{gotovos2018discrete} introduce a new sampling strategy for accelerating mixing in PSMs.

\looseness -1 Integer and continuous submodular optimization problems attract considerable attention recently. \citet{gottschalk2015submodular} present a deterministic algorithm to maximize non-monotone submodular functions on a bounded lattice; \citet{soma2014optimal} consider maximizing monotone integer submodular functions with a knapsack constraint; and \citet{soma2017non} study the integer submodular cover problem with applications on sensor placement with discrete energy levels. \citet{qian2018multiset} consider submodular maximization problems subject to size constraints while relaxing the submodularity assumption.
\citet{bian2017guaranteed} characterize the notion of continuous submodularity and present an optimal algorithm for monotone DR-submodular maximization.
\citet{hassani2017gradient} show that the projected gradient ascent algorithm achieves a 1/2 approximation for maximizing monotone DR-submodular functions.
\citet{biannips2017nonmonotone} present the local-global relation and guaranteed  algorithms for non-monotone DR-submodular maximization. Recently, \citet{bianmf,niazadeh2018optimal} propose optimal algorithms for non-monotone DR-submodular maximization with a box constraint. 
Continuous submodular maximization is also well studied in the stochastic setting \citep{hassani2017gradient,mokhtari2018stochastic} and online setting \citep{chen2018online}.


\section{Conclusion}
\looseness -1 We considered the problem of variational inference in Probabilistic Integer Submodular Models and introduced a novel continuous extension for integer submodular functions. It can be viewed as an expectation under fully-factorized marginals and we proved that it is DR-submodular even if the integer function is just submodular. Then we introduced Probabilistic Integer Submodular Models and showed that ELBO becomes DR-submodular if we use a fully-factorized categorical distribution as the variational distribution. Finally, we presented an efficient block coordinate ascent algorithm to optimize the ELBO and showed the effectiveness of our method in real world graph mining applications.

\bibliography{refs}
\bibliographystyle{icml2020}



\newpage
\newpage
\onecolumn
\appendix
\section*{Appendix}


\section{Omitted Proofs}
\label{appendix_proofs}

\begin{proposition*}
If $q(x, \theta)$ is a probability density function such that $q_{x}''(x, \theta) = q_{\theta}''(x, \theta)$, then we have
\begin{equation}
    L''(\theta) = \Int_{-\infty}^{+\infty} q_{\theta}''(x, \theta) f(x) dx = \Int_{-\infty}^{+\infty} q(x, \theta) f''(x) dx
\end{equation}
for sufficiently smooth $f$. This also holds in higher dimensions under the mean-field assumption.
\end{proposition*}

\begin{proof}
Integrate by parts twice. First, you have
\begin{align*}
    \Int_{-\infty}^{+\infty} q(x, \theta) f''(x) dx = q(x, \theta) f'(x) \Big]_{x=-\infty}^{x=+\infty} - \Int_{-\infty}^{+\infty} q_{x}'(x, \theta) f'(x) dx
\end{align*}
using $u = q(x, \theta)$ and $v = f'(x)$. Then integrate by part again 
\begin{align*}
    \Int_{-\infty}^{+\infty} q(x, \theta) f''(x) dx = q(x, \theta) f'(x) \Big]_{x=-\infty}^{x=+\infty} - q_{x}'(x, \theta) f(x) \Big]_{x=-\infty}^{x=+\infty} + \Int_{-\infty}^{+\infty} q_{x}''(x, \theta) f(x)dx
\end{align*}
using $u = q_{x}'(x, \theta)$ and $v = f(x)$
\end{proof}

\begin{proposition*}
$\text{ELBO}(\rhob)$ is DR-Submodular in $\rhob$.
\end{proposition*}
\begin{proof}
We already proved that $\frac{\partial^2 F }{\partial \rho_{ij} \partial \rho_{kl}} = 0$ when $i=k$ for all $j,l$ and $\frac{\partial^2 F }{\partial \rho_{ij} \partial \rho_{kl}} \leq 0$ when $i \neq k$ for all $j,l$.  On the other hand, $\frac{\partial^2 H_T }{\partial \rho_{ij} \partial \rho_{kl}} \leq 0$ when $i=k$ for all $j,l$ and $\frac{\partial^2 H_T }{\partial \rho_{ij} \partial \rho_{kl}} = 0$ when $i \neq k$ for all $j,l$ since $H(\rhob_i) = -(\rho_{i1} \log \rho_{i1} + \rho_{i2} \log \rho_{i2} + \ldots + (1-\rho_{i1} - \ldots - \rho_{i, k-1}) \log (1-\rho_{i1} - \ldots - \rho_{i, k-1}) )$. Therefore $\frac{\partial^2 \text{ELBO}}{\partial \rho_{ij} \partial \rho_{kl}} \leq  0$ for all $i,j,k,l$.
\end{proof}

\begin{proposition*}
Considering maximizing $\text{ELBO}(\rhob)$ in \cref{eq_elbo_objective}, if one only optimize for $\rhob_{i}$ while keeping all other marginals fixed, we have the following closed form solution (let $\nabla_{ij}:=\nabla_{\rho_{ij} }F(\rhob)  $ for notational simplicity): 

\begin{align}
\rho_{ij} = \frac{\exp{( \nabla_{ij} )} }{1 +    \sum_{j'}\exp{( \nabla_{ij'} )}}, \forall j \in \{1, ..., k-1\}.
\end{align}

\end{proposition*}

\begin{proof}
Firstly notice that the generalized multilinear extension $F(\rhob)$ is linear in terms of each $\rho_{ij}$, and it is separable for $\rho_{i1}, \rho_{i2}, ..., \rho_{i,k-1} $ for a fixed $i$. 
So if we fix all of the other marginals except for $\rhob_{i}$, $\nabla_{ij}$ will be a constant for all $j \in [k-1]$. 

Secondly the entropy term $H(\rhob)$ is concave in terms of $\rho_{ij}$,  so the $\text{ELBO}(\rhob)$
is concave in terms of $\rho_{ij}$. In order to find the maximizer of this $(k-1)$ dimensional concave function, we just need to set $\nabla_{\rhob_i} \text{ELBO}(\rhob)$ to be zero.
One can verify that 
\begin{align}
& \nabla_{ij} \text{ELBO}(\rhob) = \nabla_{ij} +
\nabla_{ij} H(\rhob)\\\notag 
& = \nabla_{ij} +  \log \frac{1 - \sum_{j'} \rho_{ij'}}{\rho_{ij} } \overset{!}{=}0, \forall j \in \{1, ..., k-1\}.
\end{align}
Solving the above $k-1$ equations, we get that $\rho_{ij} = \frac{\exp{( \nabla_{ij} )} }{1 +    \sum_{j'}\exp{( \nabla_{ij'} )}}$. 

Lastly, with this update rule, the simplex constraints are always satisfied. Because after this update, we have 
\begin{align}
 \sum_{j'} \rho_{ij'} & 
= \frac{\sum_{j'} \exp{(\nabla_{ij} )}  }{1 +\sum_{j'} \exp{(\nabla_{ij} )}} \\ 
& < 1.
\end{align}
\end{proof}

\section{More on Experiments}
\label{appendix_experiments}

The graph datasets and corresponding experimental parameters are documented in the following table: 

 \begin{table}[htbp]
 	\begin{center}
\begin{tabularx}{\textwidth}{|r|X|X|X|X|}
 			\hline
 			Dataset &   $n$  & \#edges & $q$ & \#categories  \\
 			\hline
 			\hline 
 			``Seventh graders'' & 29 & 376 &  0.7 & 6   \\
 			\hline
 			``Highschool'' & 70 & 366 & 0.2  &  10 \\
 			\hline 	
 		    ``Reality Mining''  & 96 & 1,086,404 (multiedge) &   0.75 &   6  \\
 			\hline
 			``Residence hall'' & 217 & 2,672 & 0.75  &  10  \\
 			\hline
 			``Infectious'' & 410 & 17,298 & 0.7  & 6  \\
 			\hline 			
 		\end{tabularx}
 		\caption{Graph datasets and corresponding experimental parameters}
 		\label{tab_dataset}
 	\end{center}
 \end{table}

In the plots for the marginals, we always observe the same behaviour: Shrunken FW gives smoother marginals than the Block CA. When we look at the trajectories, we see that Block CA always obtains the highest ELBO value and converges the fastest. Shrunken FW usually obtains slower convergence than Two Phase FW and obtains lower ELBO values.



\begin{figure*}[htbp]
\begin{subfigure}[t]{0.33\textwidth}
    \includegraphics[width=\linewidth]{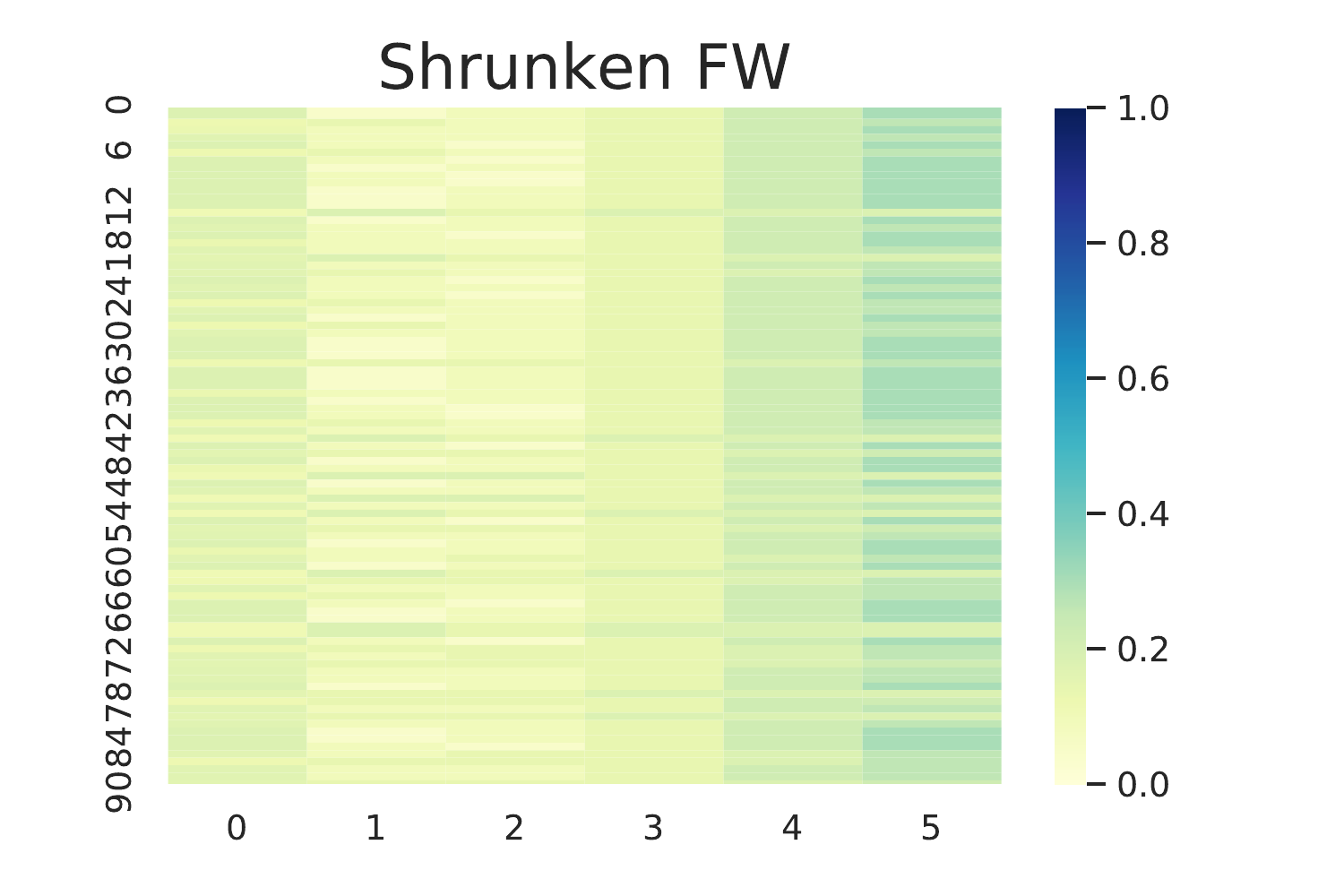}
\end{subfigure}\hfill
\begin{subfigure}[t]{0.33\textwidth}
  \includegraphics[width=\linewidth]{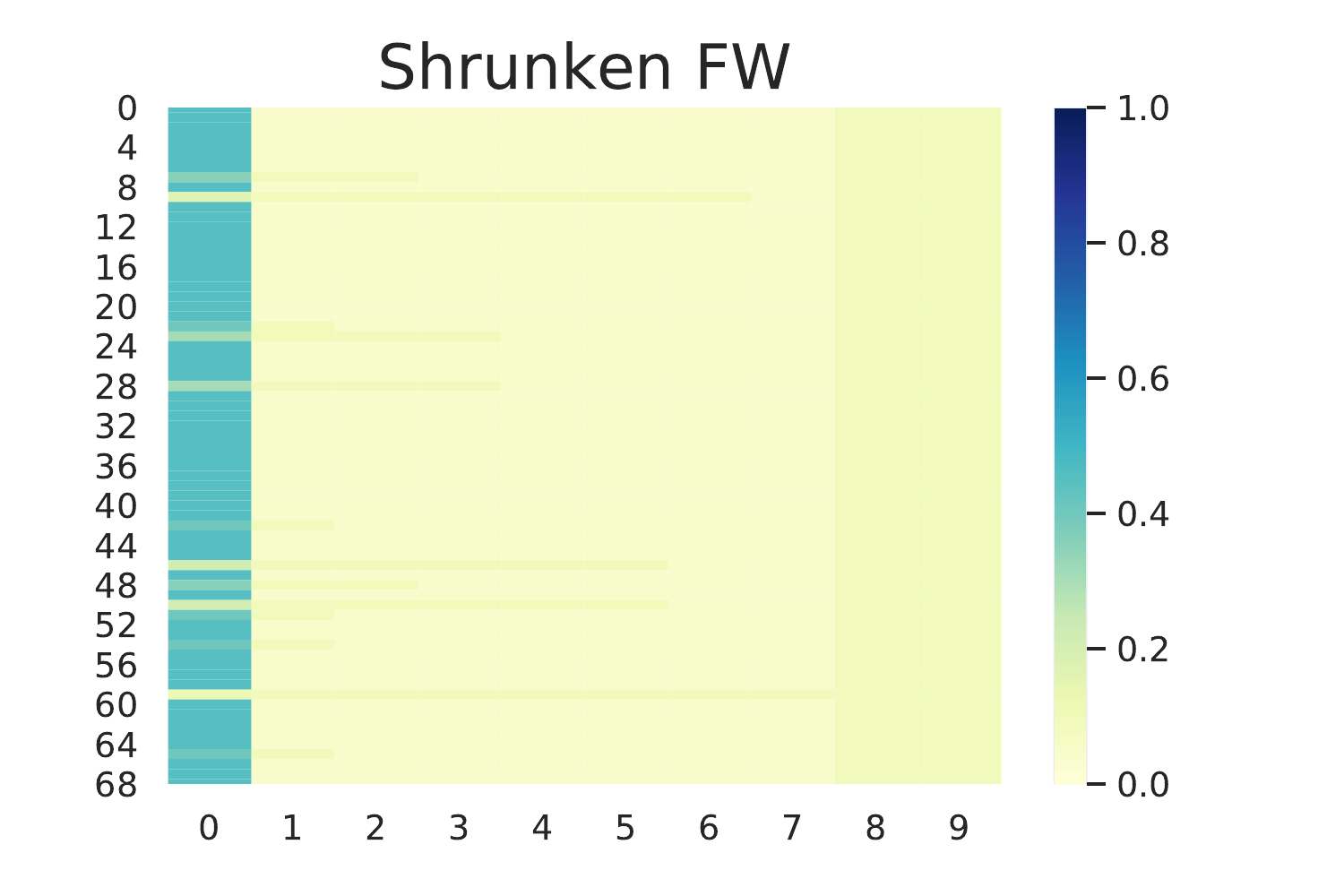}
\end{subfigure}\hfill
\begin{subfigure}[t]{0.33\textwidth}
    \includegraphics[width=\linewidth]{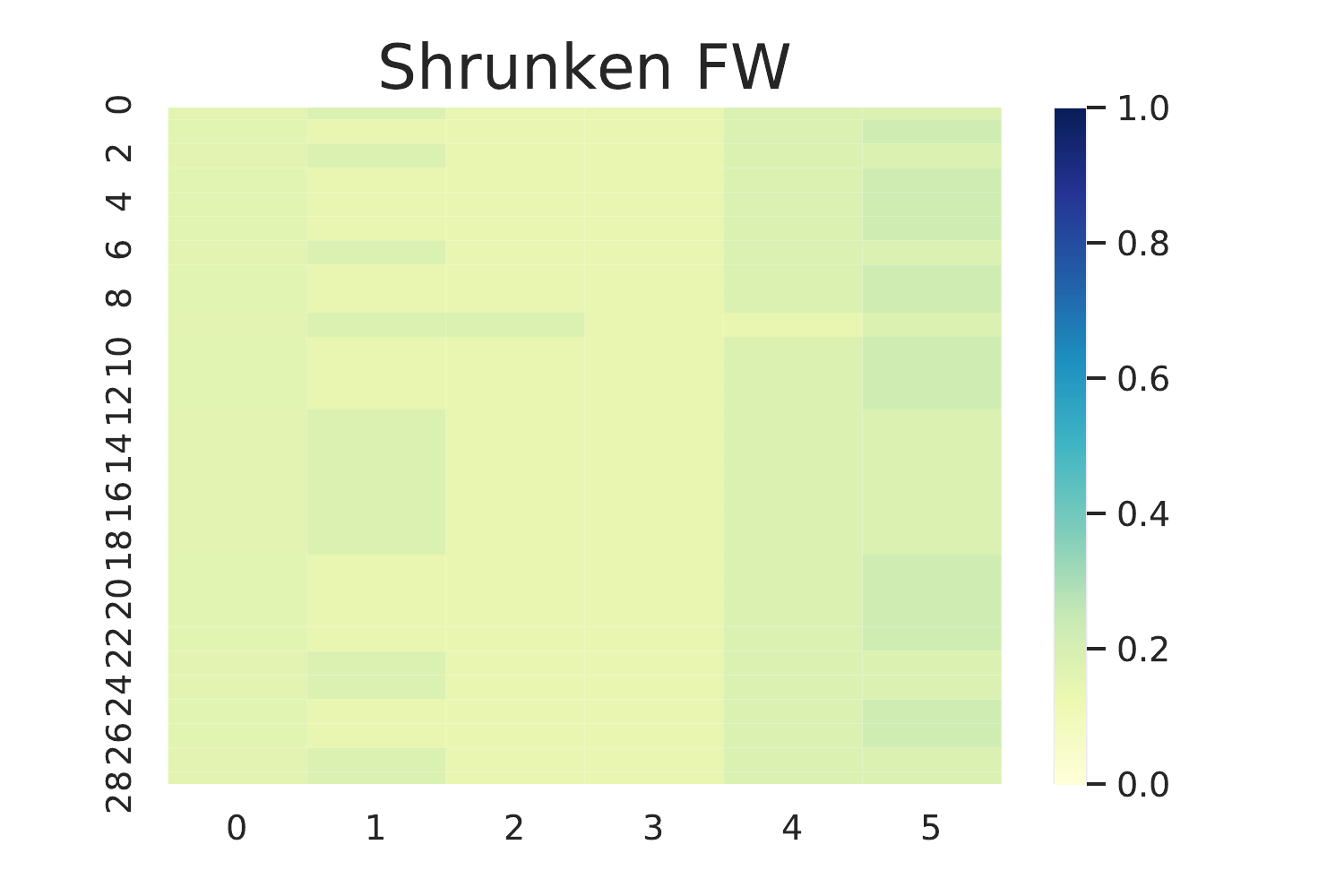}
\end{subfigure}
\begin{subfigure}[t]{0.33\textwidth}
    \includegraphics[width=\linewidth]{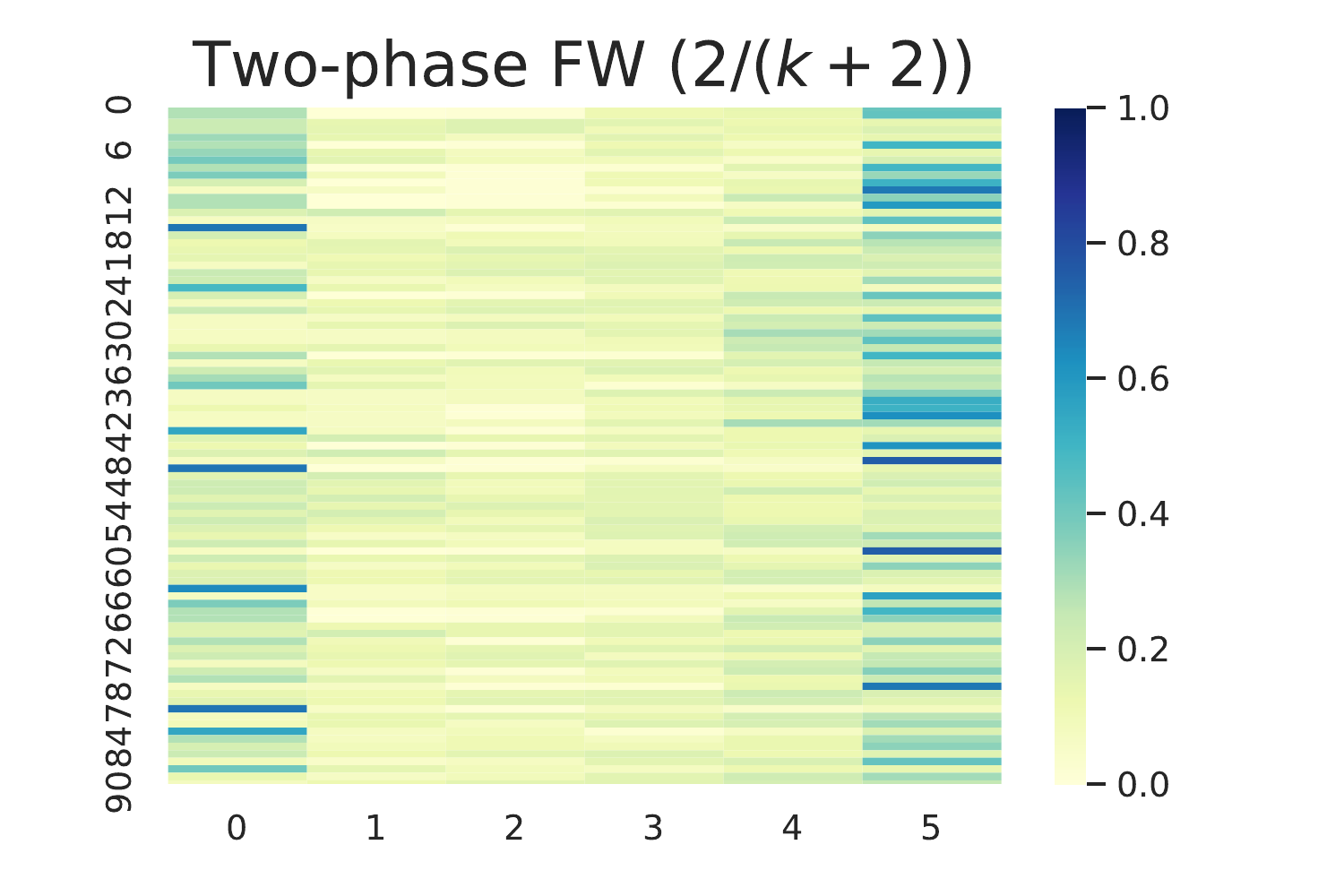}
\end{subfigure}\hfill
\begin{subfigure}[t]{0.33\textwidth}
  \includegraphics[width=\linewidth]{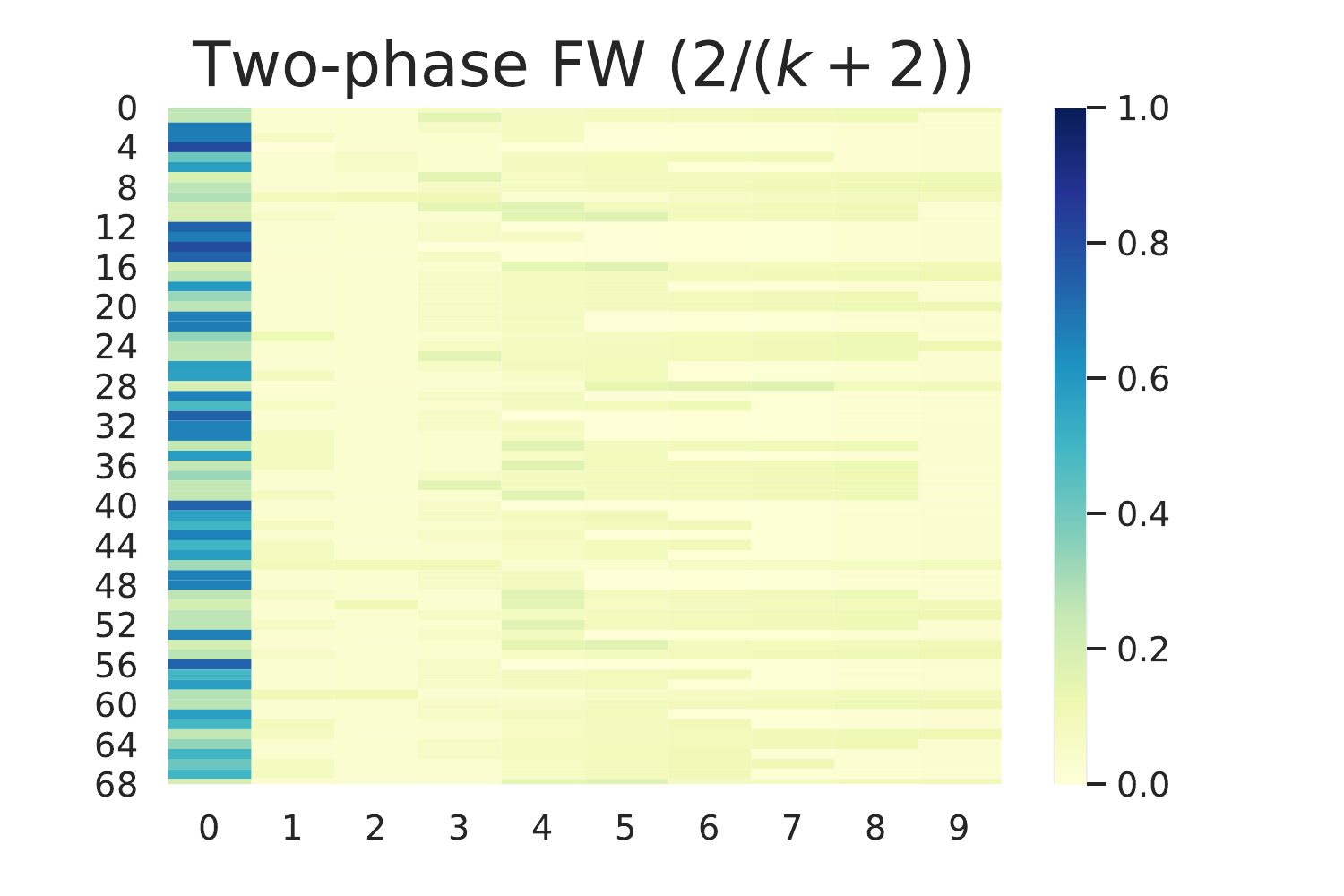}
\end{subfigure}\hfill
\begin{subfigure}[t]{0.33\textwidth}
    \includegraphics[width=\linewidth]{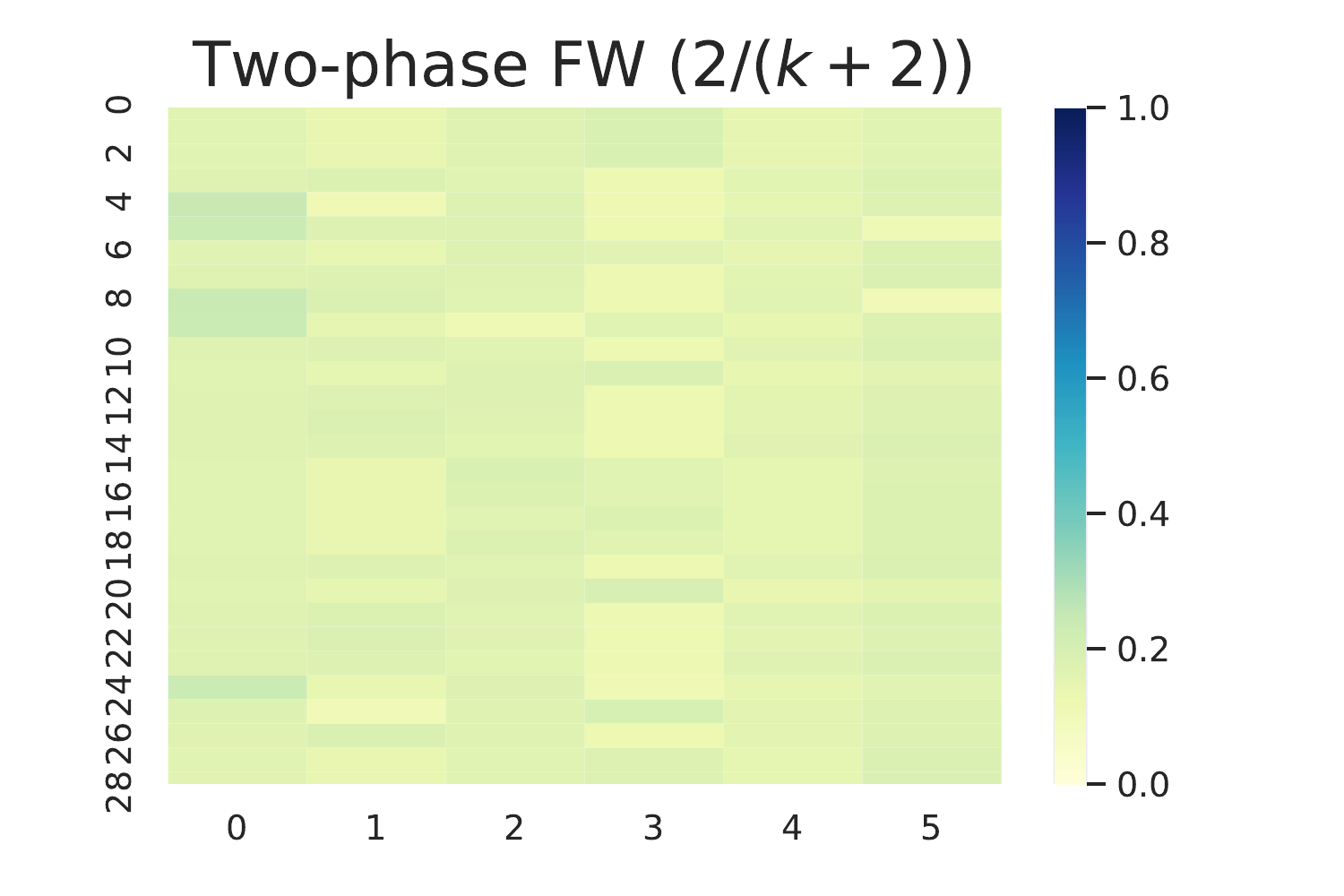}
\end{subfigure}

\begin{subfigure}[t]{0.33\textwidth}
    \includegraphics[width=\linewidth]{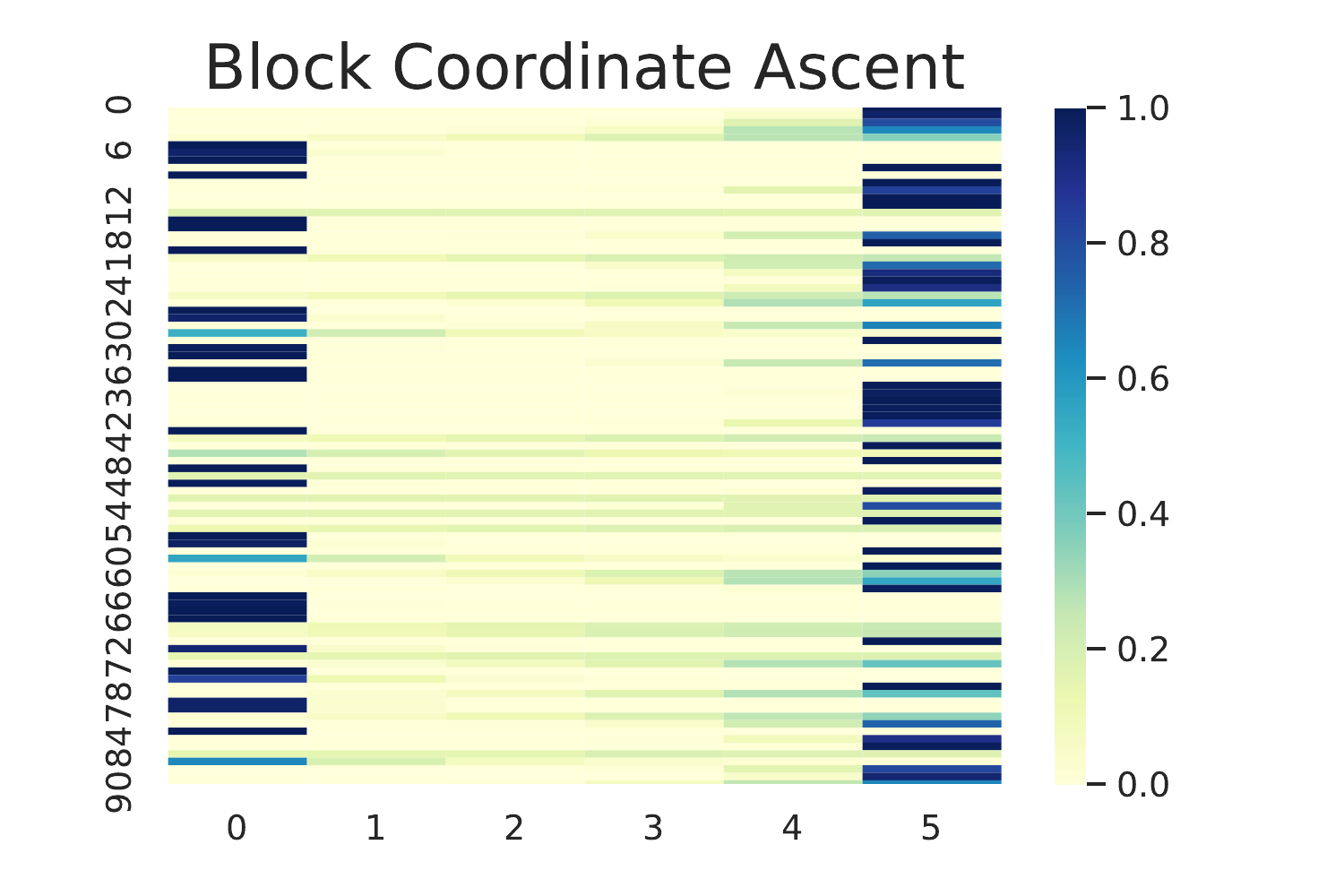}
\end{subfigure}\hfill
\begin{subfigure}[t]{0.33\textwidth}
  \includegraphics[width=\linewidth]{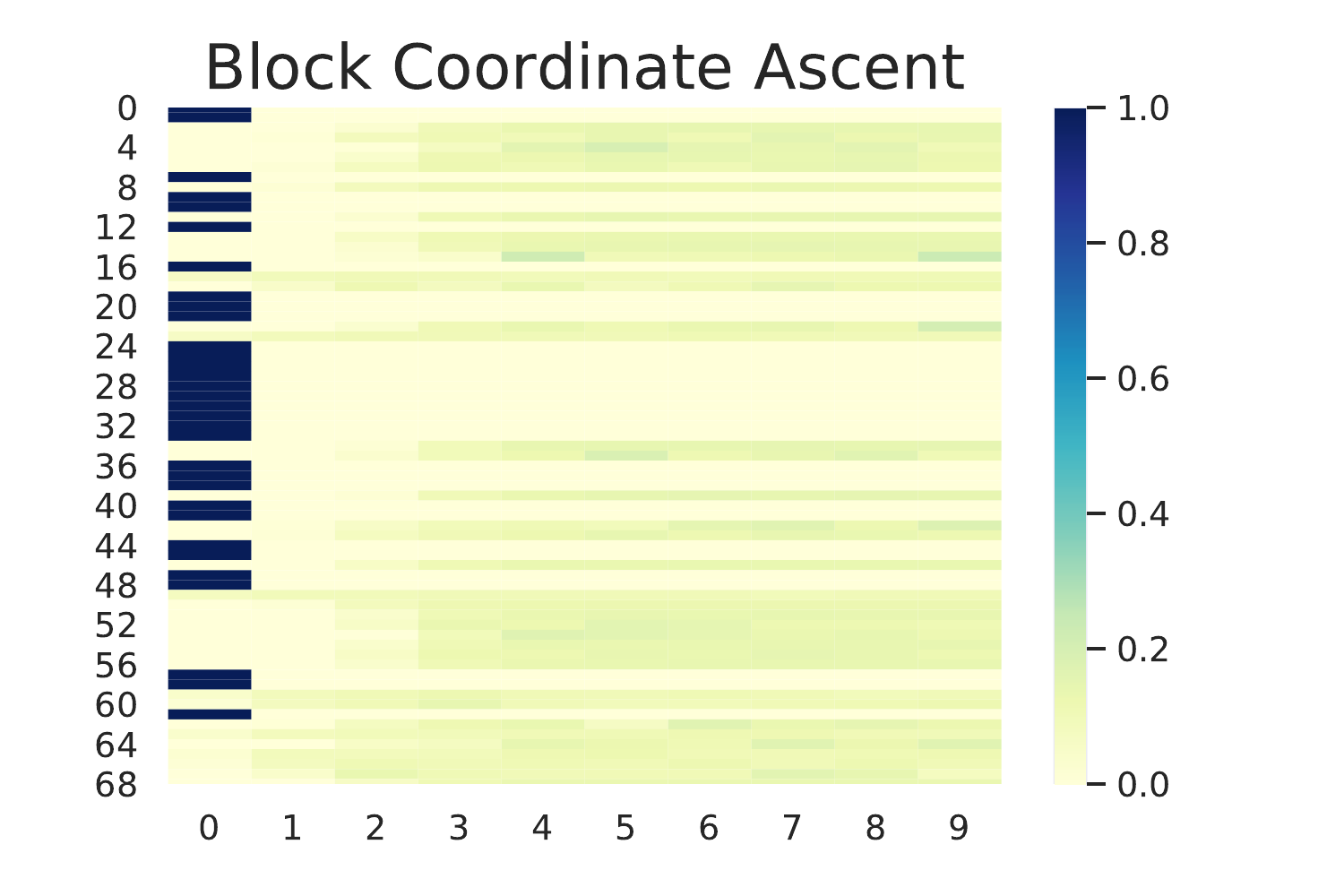}
\end{subfigure}\hfill
\begin{subfigure}[t]{0.33\textwidth}
    \includegraphics[width=\linewidth]{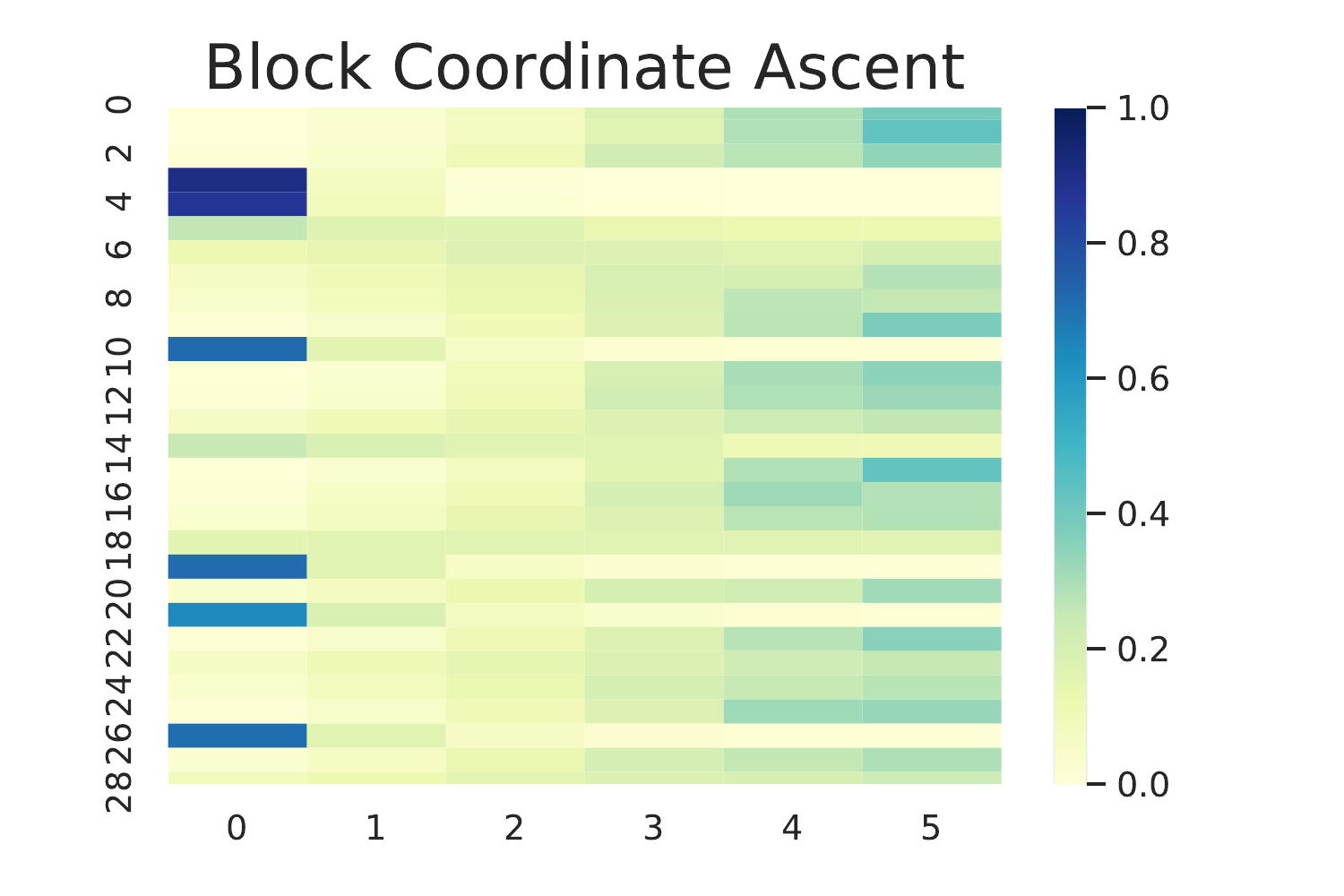}
\end{subfigure}
\begin{subfigure}[t]{0.33\textwidth}
    \includegraphics[width=\linewidth]{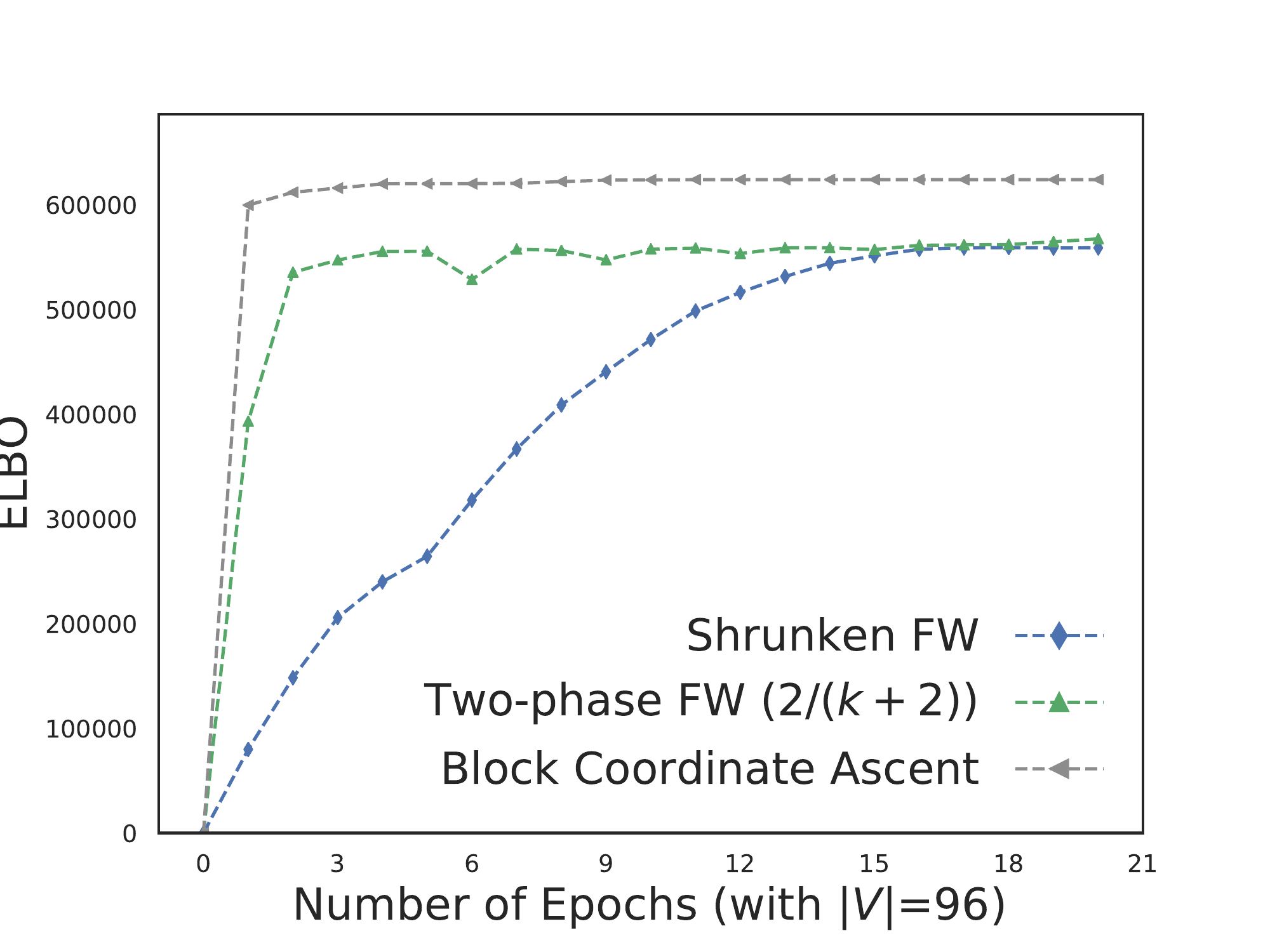}
 \caption{ ``Reality Mining'' dataset $n=96$.}
  \label{fig_Reality}
\end{subfigure}\hfill
\begin{subfigure}[t]{0.33\textwidth}
  \includegraphics[width=\linewidth]{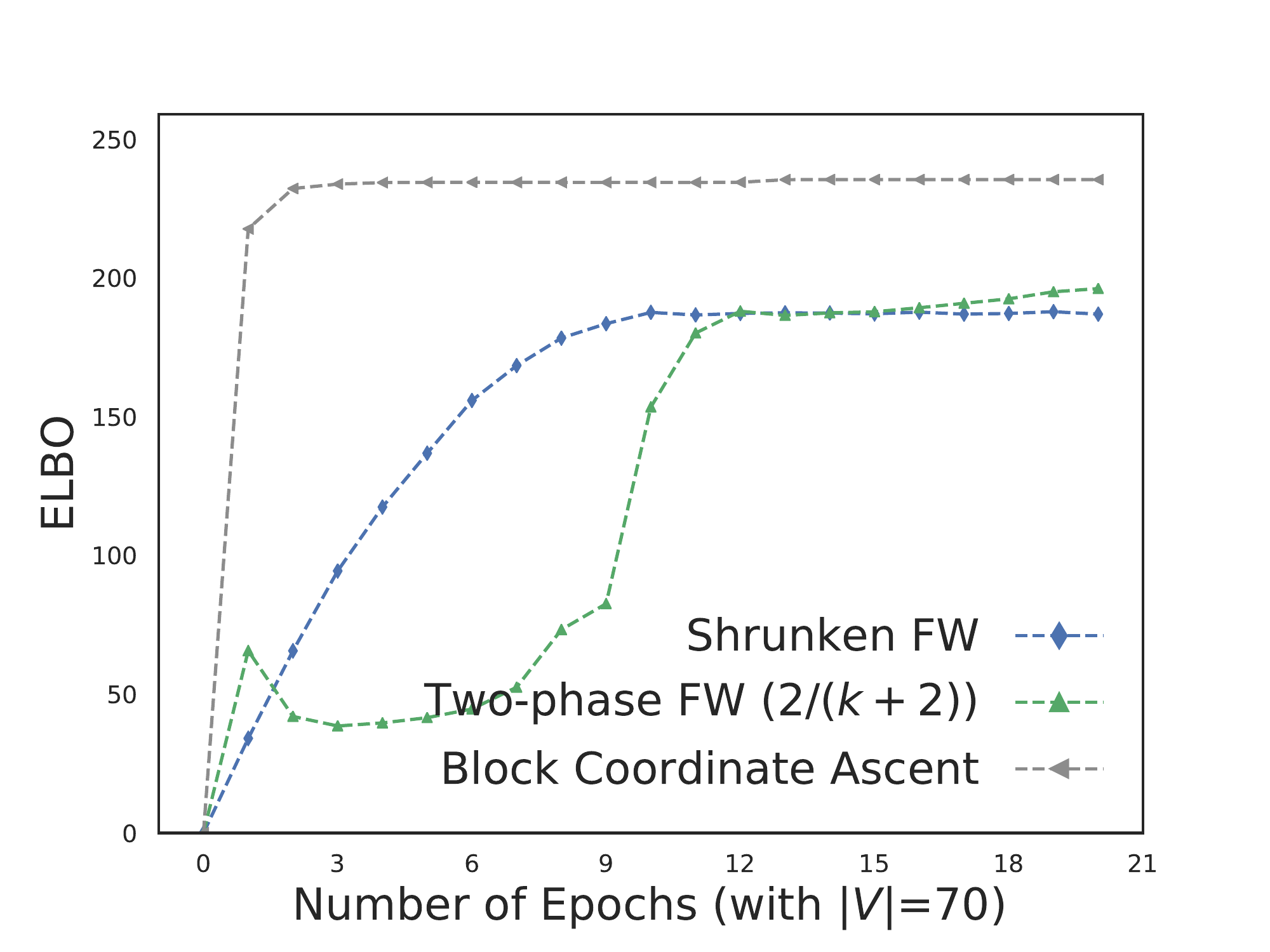}
  \caption{``Highschool'' dataset $n=70$.}
  \label{fig_Highschool}
\end{subfigure}\hfill
\begin{subfigure}[t]{0.33\textwidth}
    \includegraphics[width=\linewidth]{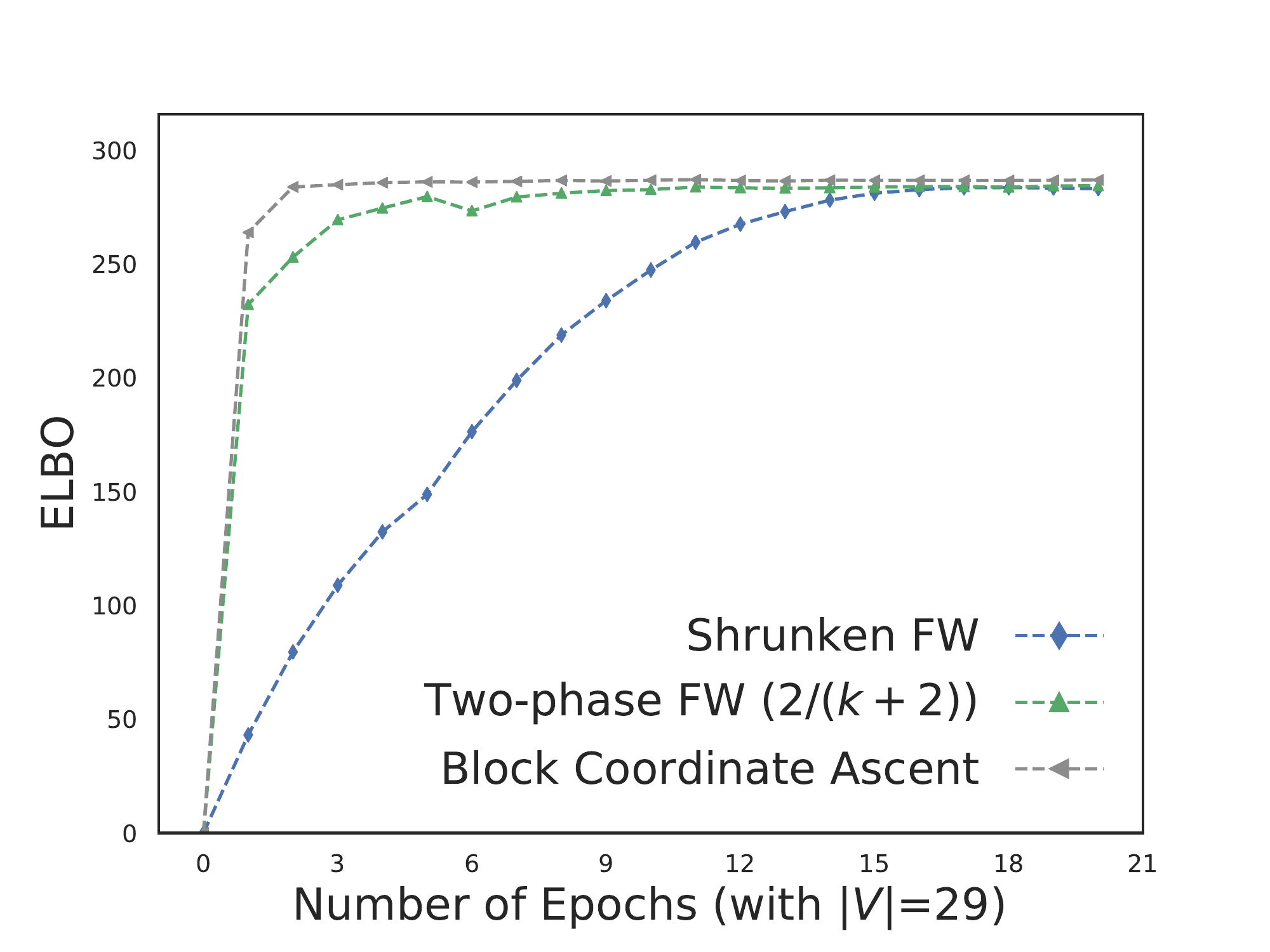}
\caption{``Seventh Graders'' dataset $n=29$.}
\label{fig_facloc}
\end{subfigure}
\caption{Marginals and Trajectories for Different Datasets and Functions}
\label{appendix_plots}
\end{figure*}

\end{document}